%% file: main.tex
\def\year{2022}\relax
\documentclass[letterpaper]{article} 
\usepackage{aaai22}  
\usepackage{times}  
\usepackage{helvet}  
\usepackage{courier}  
\usepackage[hyphens]{url}  
\usepackage{graphicx} 
\usepackage{natbib}  
\usepackage{caption} 
\DeclareCaptionStyle{ruled}{labelfont=normalfont,labelsep=colon,strut=off} 
\usepackage{algorithm}
\usepackage{algorithmic}
\usepackage[utf8]{inputenc} 
\usepackage[T1]{fontenc}    
\usepackage{url}            
\usepackage{booktabs}       
\usepackage{amsfonts}       
\usepackage{nicefrac}       
\usepackage{microtype}      
\usepackage{xcolor}   
\usepackage{tabularx}

\usepackage{multirow}
\usepackage{adjustbox}
\usepackage{gensymb}
\usepackage{amssymb}
\usepackage{times}
\usepackage{color}
\usepackage{amsthm}
\usepackage{enumitem}
\usepackage{amsmath,stmaryrd,pict2e,picture}
\usepackage{graphicx,booktabs,array}
\usepackage{nicefrac}
\usepackage{subcaption}
\usepackage{bbm}
\usepackage{comment}
\usepackage{float}
\usepackage{multirow}
\usepackage{diagbox}
\usepackage{listings}
\usepackage{url} 
\newcolumntype{C}{>{\centering\arraybackslash}m{11em}}
\definecolor{codegreen}{rgb}{0,0.8,0}
\definecolor{codegray}{rgb}{0.5,0.5,0.5}
\definecolor{codepurple}{rgb}{0.58,0,0.12}
\definecolor{backcolour}{rgb}{0.95,0.95,0.92}

\lstdefinestyle{mystyle}{
    backgroundcolor=\color{backcolour},   
    commentstyle=\color{codegreen},
    keywordstyle=\color{magenta},
    numberstyle=\tiny\color{codegray},
    stringstyle=\color{codepurple},
    basicstyle=\ttfamily\footnotesize,
    breakatwhitespace=false,         
    breaklines=true,                 
    captionpos=b,                    
    keepspaces=true,                 
    numbers=left,                    
    numbersep=5pt,                  
    showspaces=false,                
    showstringspaces=false,
    showtabs=false,                  
    tabsize=2
}
\usepackage[capitalise]{cleveref}
\newtheorem{theorem}{Theorem}
\newtheorem{definition}{Definition}

\newtheorem{proposition}{Proposition}
\newtheorem{corollary}{Corollary}

\newcommand{\eg}{\textit{e.g. }}
\newcommand{\ie}{\textit{i.e. }}

\newcommand{\argmax}{\text{arg}\max}

\usepackage[colorinlistoftodos,bordercolor=orange,backgroundcolor=orange!20,linecolor=orange,textsize=scriptsize]{todonotes}
\newcolumntype{P}[1]{>{\centering\arraybackslash}p{#1}}
\newcommand{\bibi}[1]{\todo[inline]{{\textbf{Bibi:} \emph{#1}}}}

\usepackage{newfloat}
\usepackage{listings}
\floatstyle{ruled}
\newfloat{listing}{tb}{lst}{}
\floatname{listing}{Listing}
\title{DeformRS: Certifying Input Deformations with Randomized Smoothing}
\author {
    Motasem Alfarra \textsuperscript{\rm 1, *},
    Adel Bibi \textsuperscript{\rm 2, *},
    Naeemullah Khan \textsuperscript{\rm 2},
    Philip H.S. Torr \textsuperscript{\rm 2},
    Bernard Ghanem \textsuperscript{\rm 1}
}
\affiliations {
    \textsuperscript{\rm 1} King Abdullah University of Science and Technology (KAUST), 
    \textsuperscript{\rm 2} University of Oxford\\
    motasem.alfarra@kaust.edu.sa, adel.bibi@eng.ox.ac.uk
}

\usepackage{bibentry}

\begin{document}

\maketitle



\begin{abstract}
Deep neural networks are vulnerable to input deformations in the form of vector fields of pixel displacements and to other parameterized geometric deformations \eg translations, rotations, etc. Current input deformation certification methods either (\textbf{i}) do not scale to deep networks on large input datasets, or (\textbf{ii}) can only certify a specific class of deformations, \eg only rotations. We reformulate certification in randomized smoothing setting for both general vector field and parameterized deformations and propose \textsc{DeformRS-VF} and \textsc{DeformRS-Par}, respectively. Our new formulation scales to large networks on large input datasets. For instance, \textsc{DeformRS-Par} certifies rich deformations, covering translations, rotations, scaling, affine deformations, and other visually aligned deformations  such as ones parameterized by Discrete-Cosine-Transform basis. Extensive experiments on MNIST, CIFAR10, and ImageNet show competitive performance of \textsc{DeformRS-Par} achieving a certified accuracy of $39\%$ against perturbed rotations in the set $[-10\degree,10\degree]$ on ImageNet. \footnote{Official Code: https://github.com/MotasemAlfarra/DeformRS.

$\,\,^*$Denotes equal contribution.}

\end{abstract}

\section{Introduction}

Deep Neural Networks (DNNs) are susceptible to small  additive input perturbations, \ie a DNN that correctly classifies $x$ can be fooled into misclassifying ($x+\delta$), even when $\delta$ is so small that $x$ and ($x+\delta$) are imperceptibly different \cite{szegedy2014, goodfellow2015adverserial}. Even worse, DNNs were shown to be vulnerable to input deformations \cite{alaifari2019adef} such as input rotations and scaling, where such deformations, unlike additive perturbations, can exist due to a slight change in the physical world. This raises a critical concern especially since DNNs are now deployed in safety critical applications, \eg self-driving cars. To address the nuisance of sensitivity to input deformations,  one would ideally seek to train DNNs that are \textit{certifiably} free from such adversaries. While there has been impressive progress towards this goal, \ie certifying input deformations, prior art suffers from the limitation of only being able to certify an individual set of deformations, \eg only rotations or only translations etc., or a small composition set of them \cite{singhRot,BalunovicCertifying,cvpr2002verification}. Only recently has a certification approach been developed for the richer class of smooth vector fields (general displacement of pixels) \cite{ruoss2021efficient}. However, all previous approaches require solving a mixed-integer or linear program, thus limiting their applicability to small DNNs on small datasets. On the contrary, the only certification methods that scale to larger networks on large datasets (\eg ImageNet) are based on randomized smoothing \cite{cohen2019certified}. However, such approaches \cite{fischer2020certified,li2020provable}, similar to many others, are limited to individual deformations, \eg only translations, or to deformations that ought to be \emph{resolvable} limiting the class of certifiable deformations.

\begin{figure*}
   \centering
\begin{tabular}{lll}
\includegraphics[width=0.85\textwidth]{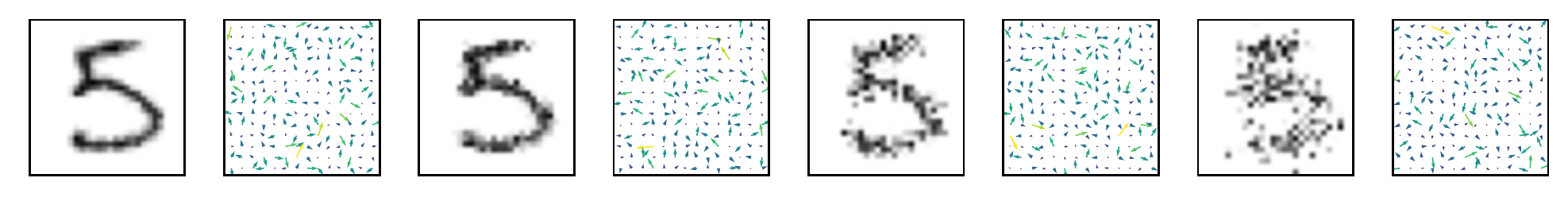} \vspace{-0.40cm}\\
\includegraphics[width=0.85\textwidth]{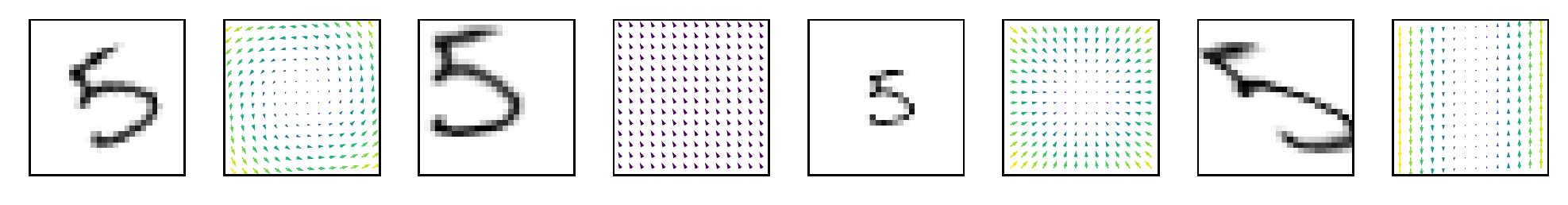}\vspace{-0.40cm} \\
 \includegraphics[width=0.85\textwidth]{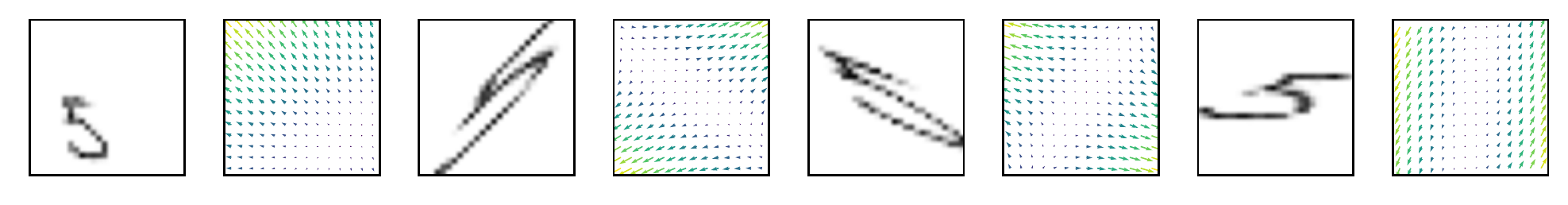} \vspace{-0.40cm}\\
\includegraphics[width=0.85\textwidth]{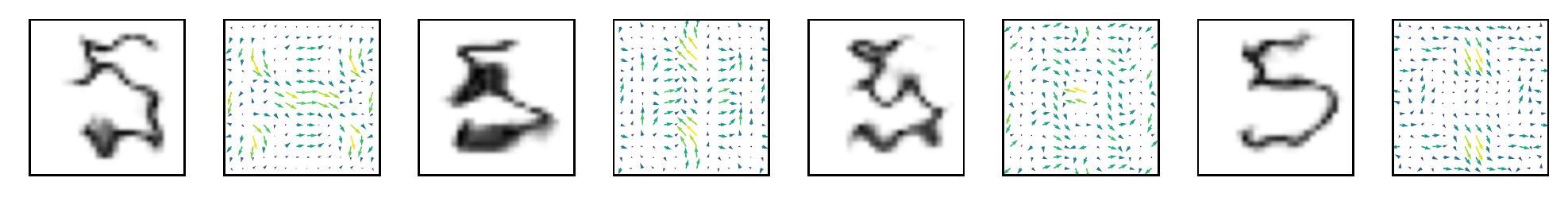} 
\end{tabular}
    \caption{\small \textbf{Examples of deformations.} We show examples of deformations accompanied with their respective vector fields. First row: Gaussian random deformations. Second row: rotation, translation, scaling, and sheering. Third row: affine deformations. Last row: DCT deformations. 
    }
    \label{fig:pull}
\end{figure*}

In this paper, we revisit the problem of certifying the parameterization of a general class of input deformations through randomized smoothing. Our approach, dubbed \textsc{DeformRS}, is general, and it allows for the certification of vector field \emph{and} parameterized deformations. For the class of parameterized deformations, \textsc{DeformRS} certifies general affine deformations that cover translation, rotations, scaling, sheering, etc., and any composition of them. Moreover, we show that if the parameterized deformation is represented by the low frequency components of the Discrete Cosine Transform (DCT), \textsc{DeformRS} allows for the certification of a set of visually aligned and plausible deformations. Figure \ref{fig:pull} presents several examples of the class of deformations \textsc{DeformRS} certifies at scale. Our contributions can be summarized as follows.
     \textbf{(i) \textsc{DeformRS-VF}.} We extend the formulation of randomized smoothing from pixel intensities to vector field deformations and derive a certification radius $R$ for the deformation vector field. That is to say, \textsc{DeformRS-VF} resists all deformations having a vector field with a norm that is smaller than $R$.
    \textbf{(ii) \textsc{DeformRS-Par.}} We specialize our analysis for parametrizable deformations and propose \textsc{DeformRS-Par}, which grants certification to popular deformations, \eg translation, rotation, scaling, and any composition subset of them, in addition to the general affine class of deformations. We also specialize \textsc{DeformRS-Par} for the set of deformations parameterized by the low-frequency components of DCT, thus certifying a richer class of visually aligned deformations that were not explored in earlier works.
    \textbf{(iii)} We demonstrate the effectiveness of our proposed approach by conducting extensive experiments on MNIST \cite{lecun1998mnist}, CIFAR10 \cite{Cifars}, and ImageNet \cite{imagenet}.
    \textsc{DeformRS-VF} is capable of providing networks that are certifiably robust against general input deformations.
    Moreover, \textsc{DeformRS-Par} achieves a certified accuracy of $96.8\%$, $91.8\%$ and $39\%$ against all rotations in the set $[-30\degree, 30\degree]$ for MNIST and $[-10\degree, 10\degree]$ on CIFAR10 and ImageNet, respectively. In comparison, a recent work \cite{cvpr2002verification} achieves a certified accuracy of $21.8\%$ on CIFAR10 under the same rotation perturbation set.

\section{Related Work}

\paragraph{Certifying Additive Perturbations.} Due to the vulnerability of DNNs to adversarial attacks \cite{szegedy2014, goodfellow2015explaining}, a stream of work was developed to build models that are certifiable against $\ell_p$ bounded additive adversaries. This includes methods based on Satisfiability Modulo Theory solvers \cite{ehlers2017formal,katz2017reluplex,bunel2017unified}, interval bound propagation \cite{ibp}, and semi-definite programming \cite{NEURIPS2018_29c0605a}, among many others \cite{ehlers2017formal, huang2017safety}. This class of approaches is generally computationally expensive for certifying deeper networks on large dimensional inputs \cite{tjeng2017evaluating} let alone for using them as part of a training routine \cite{weng2018towards}. Recently, randomized smoothing \cite{lecuyer2019certified, cohen2019certified} demonstrated to be an effective and scalable approach for probabilistic certification of additive perturbations. Followed by various improvements through incorporating adversarial training \cite{salman2019provably}, regularization \cite{zhai2020macer}, smoothing distribution optimization \cite{alfarra2020data}, randomized smoothing showed to achieve state-of-the-art performance in constructing highly accurate and certifiable networks.
Following the favorable properties of randomized smoothing, we leverage it for input deformation certification.

\paragraph{Certifying Image Deformations.} In addition to  additive input perturbations, DNNs were shown to be susceptible to input deformations. For instance, it was shown that DNNs can be fooled into mispredicting inputs undergoing small imperceptible vector field deformations (pixel displacements) \cite{alaifari2019adef}. This was followed by several works that aim to provide empirical evaluation of robustness against such deformations, \eg input translations and rotations, including attacks and defenses \cite{kanbak2017analysis, wassersteinattacksWong,engstrom2019exploring}. Unlike certification of additive input perturbations, certifying input deformations only recently started gaining attention. One of the earliest work performs an abstract interval bound propagation for certification \cite{singhRot}, which was later followed by a tighter linear program formulation \cite{BalunovicCertifying}, which certifies geometric transformations such as translation and rotation. Recently, several popular geometric transformations as well as other transformations, such as intensity contrast, were formulated as a piece-wise nonlinear layer \cite{cvpr2002verification}, thus allowing for exact certification based on a tighter formulation of classical $\ell_p$ certification solvers commonly used for additive perturbations. Moreover, recent work \cite{ruoss2021efficient} generated optimal intervals and certify them for general vector fields deformations. However, all previous methods either 
inherently 
suffer from scalability limitations, 
or that they cannot certify a composition of transformations jointly. Alleviating the scalability constraints, randomized smoothing was deployed to certify image transformations that are invariant to interpolation \cite{wassersteinsmoothing}; however, the proposed formulation was restricted to individual transformations like rotation and translation. This was followed by the work of \cite{li2020provable}, where networks were verified against individual resolvable transformations by estimating their Lipschitz upper bound. We extend prior art to allow for scalable certification of vector field deformations.

\section{Certifying Deformations with Randomized Smoothing}

\textbf{Background.} Randomized smoothing constructs a provably robust classifier $g: \mathbb R^n \rightarrow \mathcal P(\mathcal Y)$ from any classifier $f: \mathbb R^n \rightarrow \mathcal P(\mathcal Y)$, where $\mathcal{P}(\mathcal{Y})$ is a probability simplex over the set of labels $\mathcal{Y}$. 
For some distribution $\mathcal{D}$, $g$ is defined as:
\[
g(x) = \mathbb E_{\epsilon\sim\mathcal D} \left[f(x+\epsilon)\right].
\]
Suppose that $g$ assigns the class $c_A$ for an input $x$, we define:
\[
p_A = {g}^{c_A}(x,p) \quad\text{ and }\quad p_B = \max_{i \neq c_A} {g}^{i}(x,p),
\]
where $g^i(x)$ is the $i^{\text{th}}$ element of $g(x)$. Then, for Gaussian smoothing, \ie $\mathcal D = \mathcal N(0, \sigma^2 I)$, $g$ outputs a fixed prediction, \ie $ g(x) = g(x+\delta)$, for any perturbation $\delta$ satisfying $\|\delta\|_2 \leq \frac{\sigma}{2}(\Phi^{-1}(p_A) - \Phi^{-1}(p_B))$ \cite{zhai2020macer}. Here, $\Phi^{-1}$ is the inverse CDF of the standard Gaussian. Moreover, for uniform smoothing, \ie $\mathcal D = \mathcal U[-\lambda, \lambda]^n$, then $g(x) = g(x+\delta)$ for any perturbation $\delta$ satisfying $\|\delta\|_1 \leq \lambda (p_A - p_B)$ \cite{yang2020randomized}. While there has been tremendous progress in robustifying networks against $\ell_p$ additive attacks, there has been far less progress towards robustness against non-additive perturbations (\eg shadowing, input deformations, etc.).

\paragraph{Threat Model.}
We focus on the rich class of spatial deformations, \ie perturbations to the pixel coordinates, which cover as a special case translation, rotation, scaling, sheering, etc. Given an input $x$ and a function parametrized by $\kappa$ that transforms $x$ into $x'$, the threat model aims at finding parameters $\kappa$ that causes $f$ to mispredict $x'$. Formally, as proposed earlier \cite{cvpr2002verification}, the threat model solves:

\begin{equation}\label{eq:threat_model}
\begin{aligned}
\min_{\kappa} \Big(f^{c_A}(x') - \max_c f^{c \neq c_A}(x') \Big) < 0, ~~ \text{s.t.}\quad d_\kappa(x, x') < \rho,
\end{aligned}
\end{equation}

\noindent where $d_\kappa(x, x')$ measures the distance in the parameter space $\kappa$ (\eg rotation angle). In this setup, the threat model can only access the parameters of the transformation function for a given input $x$. This important formulation is studied earlier in the literature as it reformulates adversarial attacks to simulators and face recognition systems (\eg attacking pose of a face) \cite{wu2020physical, Hamdi_Mueller_Ghanem_2020}. Here, we leverage randomized smoothing to provide simple general scalable certificates against this threat model.

\subsection{\textsc{DeformRS-VF}: Certifying Vector Fields}

\paragraph{Deformations.} Let the discrete grid $\Omega^{\mathbb{Z}} \subset \mathbb{Z}^2$, where $\mathbb{Z}$ is the set of integers, represent the domain of images $I:\Omega^{\mathbb{Z}} \rightarrow [0,1]^{c}$, where $c$ is the number of channels in the image. Then, a domain deformation  is defined as $T:\Omega^{\mathbb{Z}} \rightarrow \mathbb{R}^2$, such that for a pixel coordinate $p \in \Omega^{\mathbb{Z}}$, we can write $T(p)=p+v(p)$, where $v: \Omega^{\mathbb{Z}} \rightarrow \mathbb{R}^2$ represents the vector field. Since the deformation $T$ maps pixel coordinates to $\mathbb{R}^2$, one needs to define an associated interpolation function for a deformed image $x \in [0,1]^n$, where $n = c \times |\Omega^{\mathbb{Z}}|$, as $I_T : [0,1]^n \times \mathbb{R}^{2|\Omega^\mathbb{Z}|} \rightarrow [0,1]^n$. As such, when $T(p) = p ~~\forall p \in \Omega^{\mathbb{Z}}$, then we have $I_T(x,T(p)) = x$. For ease of notation, we use $p$ to denote the complete set of the discrete grid $\Omega^\mathbb{Z}$. First, we extend the definition of smoothed classifiers to domain deformation smoothed classifiers.

\begin{definition}\label{def:smooth_domain_classifier}
Given a classifier $f : \mathbb R^n \rightarrow \mathcal P(\mathcal Y)$ and an interpolation function $I_T : [0,1]^n \times \mathbb{R}^{2|\Omega^\mathbb{Z}|} \rightarrow [0,1]^n$, we define a deformation smoothed classifier as:
\begin{align*}
    \hat{g}(x,p) = \mathbb{E}_{\epsilon \sim \mathcal{D}}
    \left[f\left(I_T(x, p + \epsilon)\right)\right].
\end{align*}
\end{definition}

Note that contrary to $g$, which smooths the predictions of $f$ under additive pixel perturbations, $\hat g$ smooths predictions of $f$ under pixel coordinate deformations. Similar in spirit to earlier results on randomized smooth for additive perturbations \cite{cohen2019certified, zhai2020macer}, we can show that $\hat{g}$ is certifiable as per the following Theorem. We leave all proofs to the \textbf{Appendix}.

\begin{theorem}
\label{theo:vector_field_certification}
Suppose that $\hat g$ assigns the class $c_A$ for an input $x$, \ie $c_A = \argmax_c \hat{g}^c(x,p)$ with:
\[
p_A = \hat{g}^{c_A}(x,p) \quad\text{ and }\quad p_B = \max_{i \neq c_A} \hat{g}^{i}(x,p)
\]
then $\argmax_c \hat g^c(x, p+\psi) = c_A$ for vector field perturbations satisfying:
\begin{equation}
\begin{aligned}\label{eq:vector_field_certification_inequality}
\begin{split}
    &\|\psi\|_1 \leq \lambda \left(p_A - p_B\right) \qquad \qquad \qquad \quad \,\,\,\, \text{for } \mathcal D = \mathcal U[-\lambda, \lambda],  \\
    &\|\psi\|_2 \leq \frac{\sigma}{2}\left(\Phi^{-1}(p_A) - \Phi^{-1}(p_B)\right) \qquad \text{for } \mathcal D = \mathcal N(0, \sigma^2I),
\end{split}
\end{aligned}
\end{equation}
\end{theorem}

Theorem \ref{theo:vector_field_certification} states that as long as the $\ell_1$ and $\ell_2$ norms of the deformation characterized by the vector field $\psi$ are sufficiently small, then $\hat g$ enjoys a constant prediction. Note that the $\ell_1$ and $\ell_2$ certificates are agnostic to the structure of the deformation vector field $\psi$. That is to say, $\hat g$ resists all domain deformations, \eg translation, rotation, scaling, etc., as long as \eqref{eq:vector_field_certification_inequality} is satisfied. This includes patch level deformations, \ie when $\psi$ is an all zero vector field except for a set of indices representing a patch (\eg a rotation of a patch in the image).

\subsection{\textsc{DeformRS-Par:} Certifying Parametrizable Deformations}

Note that the dimensionality of the deformation vector field $\psi$ is twice (two dimensions of the image) the number of pixel coordinates, \ie $2 |\Omega^\mathbb{Z}|$, where $|\Omega^\mathbb{Z}|= 32 \times 32$ in CIFAR10. As such, the set of deformation vector fields $\psi$ of this large dimensionality satisfying the conditions in \eqref{eq:vector_field_certification_inequality} might be limited to a set of imperceptible deformations, \ie $x$ and $I_T(x,p+\psi)$ are indistinguishable
\footnote{Certifying imperceptible deformations is important since adversaries can take this form \cite{alaifari2019adef}.}.
However, many popular deformations are parameterized by a much smaller set of parameters. In general, consider the deformation $T_\phi(p) = p + v_\phi(p)$, where the dimension of $\phi$ is much lower than $v_\phi(p)$, where $v_\phi$ is an element wise function. For example, when the vector field $v_\phi$ characterizes a translation or a rotation, the parameterization $\phi$ is of dimensions 2 and 1, respectively. In that regard, we show that a close relative to Theorem \ref{theo:vector_field_certification} also holds for perturbations in the parameters characterizing deformations. We first define a parametric deformation smoothed classifier.

\begin{definition}\label{def:smooth_parametric_classifier}
Given a classifier $f:\mathbb{R}^n \rightarrow \mathcal{P}(\mathcal{Y})$ and an interpolation function $I_T : [0,1]^n \times \mathbb{R}^{2|\Omega^\mathbb{Z}|} \rightarrow [0,1]^n$, we define a parametric deformation smoothed classifier as follows:
\begin{align*}
    \tilde{g}_\phi(x, p) = \mathbb{E}_{\epsilon \sim \mathcal D}
    \left[f\left(I_T\left(x,  p + v_{\phi+\epsilon}(p)\right)\right)\right] .
\end{align*}
\end{definition}

Unlike Definition \ref{def:smooth_domain_classifier}, $\tilde{g}_\phi$ smooths the prediction of $f$ under a specific class of deformations by perturbing the parameterization $\phi$. Next, we analyze the robustness of $\tilde{g}_\phi$. 

\begin{corollary}\label{cor:parametric_certification}
Suppose that $\tilde g$ assigns the class $c_A$ for an input $x$, \ie $c_A = \text{arg} \max_c \tilde{g}_\phi(x,p)$ with:
\[
p_A = \tilde{g}_\phi^{c_A}(x,p) \quad\text{ and }\quad p_B = \max_{i \neq c_A} \tilde{g}_\phi^{i}(x,p),
\]
then $\text{arg}\max_c \tilde g_{\phi+\xi}(x, p) = c_A$ for all parametric domain perturbations satisfying:
\begin{align*}
\begin{split}
    &\|\xi\|_1 \leq \lambda \left(p_A - p_B\right) \qquad \qquad \qquad \quad \,\,\,\, \text{for } \mathcal D = \mathcal{U}[-\lambda, \lambda], \\
    &\|\xi\|_2 \leq \frac{\sigma}{2}\left(\Phi^{-1}(p_A) - \Phi^{-1}(p_B)\right) \qquad \text{for } \mathcal D = \mathcal N(0, \sigma^2I).
\end{split}
\end{align*}

\end{corollary}
Corollary \ref{cor:parametric_certification} specializes the result of Theorem \ref{theo:vector_field_certification} to the family of parametric deformations. It states that as long as the norm of the perturbations to the deformation parameters is sufficiently small, $\tilde{g}_\phi$ enjoys a constant prediction. Next, we show the parametrization of several popular deformations. Let the pixel grid $\Omega^{\mathbb{Z}}$ be the grid of an image of size $N\times M$, where $p_{n,m} = (n,m)\in \Omega^{\mathbb{Z}}$ is a pixel location and  $v_\phi(p_{n,m}) = (u_{n,m}, v_{n,m})$ represents the  field at $p_{n,m}$.

\paragraph{Translation.} Image translation is only parameterized by two parameters $\phi = \{t_v,\, t_v\}$, namely $v_\phi(p_{n,m}) = (t_u, t_v) ~~\forall p$ $~\forall n,m$ as per Definition \ref{def:smooth_parametric_classifier} and Corollary \ref{cor:parametric_certification}. 

\paragraph{Rotation.} 2D rotation is only parameterized by the rotation angle $\phi = \{\theta\}$, where $u_{n,m} = n (cos(\theta)-1) - m sin(\theta)$ and $v_{n,m} = n sin(\theta) + m (cos(\theta)-1)$.

\paragraph{Scaling.} Similar to rotation, scaling is parametrized with one parameter; the scaling factor $\phi = \{\alpha\}$, where $u_{n,m} = (\alpha-1)n$ and $v_{n,m} = (\alpha-1)m$ $\forall n,m$. That is to say, the vector field has the form $v_\alpha(p) = ((\alpha-1)n,(\alpha-1)m)~\forall p$.

\paragraph{Affine.} Our formulation for the certification of the parametric family of deformations is general and covers all affine vector fields as special cases. In particular, affine vector fields are parameterized by 6 parameters, namely $\phi = \{a,b,c,d,e,f\}$, where $u_{n,m} = an + bm + e$ and $v_{n,m} = cn + dm + f$. Note that this class naturally covers composite deformations, such as scaling and translation jointly.

\paragraph{Beyond Affine: DCT- Basis.} To address deformations beyond affine vector fields, we also consider certifying a class of deformations represented by the Discrete Cosine Transform (DCT) basis. In particular, we consider the low-frequency component truncated DCT of the vector field $u_{n,m}$ and $v_{n,m}$ with a window size of $k \times k$ 
(as opposed to the complete size of $N \times M$), where the set is characterized by $2k^2$ 
parameters. 

\section{Experiments}

We validate the certified performance of \textsc{DeformRS} following Theorem \ref{theo:vector_field_certification} and Corollary \ref{cor:parametric_certification}, respectively. The goal of this section is to show that (\textbf{i}) \textsc{DeformRS-Par} improves certified accuracy against individual deformations that are parameterizable, \eg rotation as compared to \cite{cvpr2002verification} (MOH), in addition to comparisons against several other individual deformations on several datasets. (\textbf{ii}) \textsc{DeformRS-Par} can certify the general class of affine deformations allowing for the certification of a composition of deformations, \eg rotation and sheering jointly. (\textbf{iii}) \textsc{DeformRS-Par} can certify deformations that are parameterized by truncated DCT coefficients, a more general class of deformations that can represent visually aligned deformations. (\textbf{iv}) Following Theorem \ref{theo:vector_field_certification}, \textsc{DeformRS-VF} certifies general vector field deformations that are generally imperceptible.
Here, we note that while our work directly compares to MOH in terms of setup and formulation (certifying parameter perturbations as per the threat model in Objective \ref{eq:threat_model}), we include other geometric certification approaches, \ie \cite{li2020provable} (LI),\cite{ BalunovicCertifying} (BAL), and \cite{fischer2020statistical} (FBV), that are not directly comparable to ours due to a different threat model for full completeness.

\begin{table*}[t]
\centering
\caption{\small\textbf{Certifying  individual deformations on MNIST and CIFAR10.} We compare the certified accuracy of \textsc{DeformRS-Par} against (R)otation, (S)caling and (T)ranslation with that of prior art. We define $\|\psi\|_2 = (t_u^2 + t_v^2)^{\nicefrac{1}{2}}$ for translation. \textbf{(a)} and \textbf{(d)}: Certified accuracy at T$(\|\psi\|_2 \leq 2)$ and T$(\|\psi\|_2 \leq 2.41)$, respectively; we adopt these settings from BAL and FBV. \textbf{(b)} and \textbf{(c)}: Certified accuracy at R$(6.79 \degree)$ and R$(38.24 \degree)$, respectively; we adopt these settings from FBV. Note that \textsc{DeformRS-Par} achieves a higher certified accuracy than \textbf{(a, b, d)} at a higher radius. Best certified accuracies are highlighted in \textbf{bold}.}
\centering
\resizebox{0.99\textwidth}{!}{
\scriptsize
\begin{tabular}{ c | ccc | ccc }
\toprule 
\multirow{2}{*}{Certification} &\multicolumn{3}{c|}{MNIST} & \multicolumn{3}{c}{CIFAR-10} \\
  & R$(30\degree)$ & S$(20\%)$ & T$(\|\psi\|_2 \leq 5)$& R$(10\degree)$ & S$(20\%)$ & T$(\|\psi\|_2 \leq 5)$\\
\midrule
  (BAL,MOH)    &87.80(BAL)  & - & 77.00$^\text{\textbf{(a)}}$(BAL) &    87.80 (BAL), 21.80 (MOH)&   -    &-  \\
 (FBV)& 72.75$^\textbf{(c)}$ & - & 95.00$^\textbf{(d)}$ & 42.00$^\textbf{(b)}$ & - & -  \\
  (LI)  & 95.60 & 96.80 & 96.80 & 63.80 & 58.40  & 84.80 \\
 \textsc{DeformRS-Par}   & \textbf{96.85}, \textbf{96.10}$^{\textbf{(c)}}$ & \textbf{98.70} & \textbf{99.20} & \textbf{91.82}     & \textbf{90.30}      & \textbf{88.80}  \\
\bottomrule
\end{tabular}\label{tb:MNIST-CIFAR10}}
\end{table*}

\begin{table}[t]
\small
\caption{\small\textbf{Certifying individual deformations on ImageNet.} We compare the certified accuracy of \textsc{DeformRS-Par} against (R)otation, (S)caling and (T)ranslation with prior art. \textbf{(e)}: Certified accuracy at R$(1.86 \degree)$; we adopt this setting from FBV.}
\centering
\begin{tabular}{ c | ccc }
\toprule 
\multirow{2}{*}{Certification} & \multicolumn{3}{c}{ImageNet} \\
  & R$(10\degree)$ & S $(15\%)$ & T$(\|\psi\|_2 \leq 5)$\\
\midrule
  (FBV)   &   17.25$^{\textbf{(e)}}$    &     - &  -\\
  (LI)  &    33.00   &    31.00   & \textbf{63.30} \\
 \textsc{DeformRS-Par}   &  \textbf{39.00}    & \textbf{42.80}      & 48.20\\
\bottomrule
\end{tabular}\label{tb:ImageNet}
\end{table}

\textbf{Setup.} We follow standard practices prior art, \eg LI and FBV,  and conduct experiments on MNIST \cite{lecun1998mnist}, CIFAR10 \cite{Cifars}, and ImageNet \cite{imagenet} datasets. For experiments on MNIST and CIFAR10, we certify ResNet18 \cite{resnet} trained for 90 epochs with a learning rate of 0.1, momentum of 0.9, weight decay of $10^{-4}$, and learning rate decay at epochs 30 and 60 by a factor of 0.1. For ImageNet experiments, we certify a fine-tuned pretrained ResNet50 for 30 epochs using SGD with a learning rate of $10^{-3}$ that decays at every 10 epochs by a factor of 0.1. All networks are trained with data augmentation sampled from the respective deformations that are being certified, so as to attain a highly accurate base classifier $f$ under such deformations. Following randomized smoothing methods \cite{salman2019provably,zhai2020macer,alfarra2020data} and using publicly available code \cite{cohen2019certified}, all our results are certified with $100$ Monte Carlo samples for the selection of the top prediction $c_A$ and $100,000$ samples for the estimation of a lower bound to the prediction probability $p_A$ with a failure probability of $0.001$. Throughout all experiments, we choose $I_T$ to be a bi-linear interpolation function. Moreover, since image dimensions vary across datasets (square images of sizes 28, 32, 224 for MNIST, CIFAR10 and ImageNet, respectively), we normalize all image dimensions to $[-1, 1]\times[-1, 1]$. While our certificate has a probabilistic nature, we compare against both mixed integer and linear program based certification methods (BAL, MOH), as well as randomized smoothing based approaches (FBV, LI) for comprehension.

\textbf{Evaluation metrics.} Following prior art (FBV, LI), we use certified accuracy to compare networks. The certified accuracy at a radius $R$ is the percentage of the test set that is both correctly classified  and has a certification radius of at least $R$. Note that $R$ is computed following Corollary \ref{cor:parametric_certification} for \textsc{DeformRS-Par} and Theorem \ref{theo:vector_field_certification} for \textsc{DeformRS-VF}. 
We report the Average Certified Radius (ACR)  \cite{zhai2020macer}.

\subsection[xx]{\textsc{DeformRS-Par} - Paramterizable Deformations
\footnote{Certifying deformations lack standard benchmarks and evaluation protocols. This is why there are several superscripts in Tables \ref{tb:MNIST-CIFAR10} and \ref{tb:ImageNet} as methods report certified accuracies at different radii.}}

\textbf{Rotation.} Rotation deformations are parameterized with a bounded scalar representing the rotation angle $\theta\in [-\pi, \pi]$. Therefore, we use the Uniform smoothing variant of Corollary \ref{cor:parametric_certification} resulting in a certification of the form $|\theta| \leq \lambda (p_A -p _B)$. We train several networks with $\lambda \in \{\nicefrac{\pi}{10}, \nicefrac{2\pi}{10}, \dots, \pi\}$, where each trained network is certified with the corresponding $\lambda$ used in training.
We compare the rotation certified accuracy of \textsc{DeformRS-Par} against that of prior work (BAL, FBV, LI, and MOH) on MNIST and CIFAR10 in Table \ref{tb:MNIST-CIFAR10} and on ImageNet in Table \ref{tb:ImageNet}. Following the common practice in randomized smoothing literature \cite{salman2019provably, zhai2020macer}, Tables \ref{tb:MNIST-CIFAR10} and \ref{tb:ImageNet} report the best certified accuracies for \textsc{DeformRS-Par} cross-validated over $\lambda$. 

In particular, and as shown in Table \ref{tb:MNIST-CIFAR10}, \textsc{DeformRS-Par} outperforms its best competitor by $1.25\%$ and $4\%$ on MNIST and CIFAR-10 at rotation radii of $30\degree$ (\ie R$(30\degree)$) and $10\degree$ (\ie R$(10\degree)$), respectively.  Interestingly, on CIFAR10, the certified accuracy of \textsc{DeformRS-Par} at radius $ 10\degree$ is even better than the accuracy of FBV reported at the smaller angle radius of $6.79\degree$. The improvement is consistent on ImageNet, where \textsc{DeformRS-Par} outperforms the randomized smoothing based approach of LI  by $6\%$, as reported in Table \ref{tb:ImageNet}. Further, we report an improvement of 70\% on the certified accuracy on CIFAR10 at radius $10\degree$ against MOH that shares the same threat model to our formulation. We believe that \textsc{DeformRS-Par} outperforms mixed-integer and linear program rotation certification methods due to their high computational cost that  results in prohibitive explicit training for improved certification (BAL, MOH). We plot in the first column of Figure \ref{fig:LowDimTransformations_MNISTCIFARImagenet} the certified accuracy of \textsc{DeformRS-Par} over a subset $\lambda$ used for training and certification. We leave the rest of the ablations of $\lambda$ to the \textbf{Appendix}. We observe that the certified accuracies of \textsc{DeformRS-Par} at the radii reported in the previous tables are indeed insensitive to the choice of $\lambda$. Moreover, we note that \textsc{DeformRS-Par} attains a certified accuracy of at least 80\% on both MNIST and CIFAR10 at a radius of $100\degree$. In addition, when $\lambda = 90\degree$ on MNIST, \textsc{DeformRS-Par} attains an ACR of $85\degree$, \ie the average certified rotation is $85\degree$. 

\textbf{Scaling.} Scaling deformations are parameterized by $\alpha \ge 0$. Note that a scaling $\alpha$ can either be a zoom-out $(\alpha > 1)$ or a zoom-in $(0 < \alpha < 1)$. For ease, we consider the bounded scaling factor $\alpha - 1$ instead such that $|\alpha - 1| < 0.7$. Thus, an appropriate smoothing distribution in Corollary \ref{cor:parametric_certification} is uniform with $\lambda \in \{0.1, 0.2, \dots, 0.7\}$ granting a certificate of the form $|\alpha - 1| \leq \lambda(p_A-p_B)$. We report the certified accuracy at the scale factor of $20\%$  $(\ie 0.8 \leq \alpha \leq 1.20)$ in Table \ref{tb:MNIST-CIFAR10} for MNIST and CIFAR10, and at a scale factor of 15\% $(\ie 0.85 \leq \alpha \leq 1.15)$ for ImageNet in Table \ref{tb:ImageNet}. The best certified accuracy cross validated over $\lambda$ for \textsc{DeformRS-Par} outperforms its best competitor (LI) by 1.9\% on MNIST, 31.9\% on CIFAR10, and 11.8\% on ImageNet. Moreover, we plot the certified accuracy in the second column of Figure \ref{fig:LowDimTransformations_MNISTCIFARImagenet} showing the insensitivity of \textsc{DeformRS-Par} to $\lambda$. Moreover, \textsc{DeformRS-Par} enjoys a certified accuracy of at least (90\%, 80\%, 40\%) at the larger scaling factors of (0.5, 0.4, 0.2) \ on MNIST, CIFAR10 and ImageNet, respectively.

\begin{figure*}[t]
    \centering
    \includegraphics[width=0.85\textwidth]{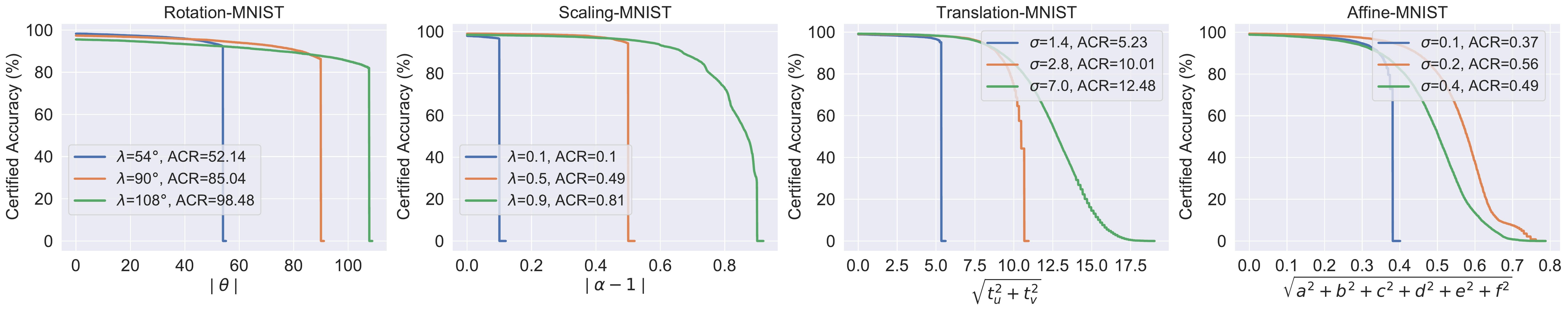}
    \includegraphics[width=0.85\textwidth]{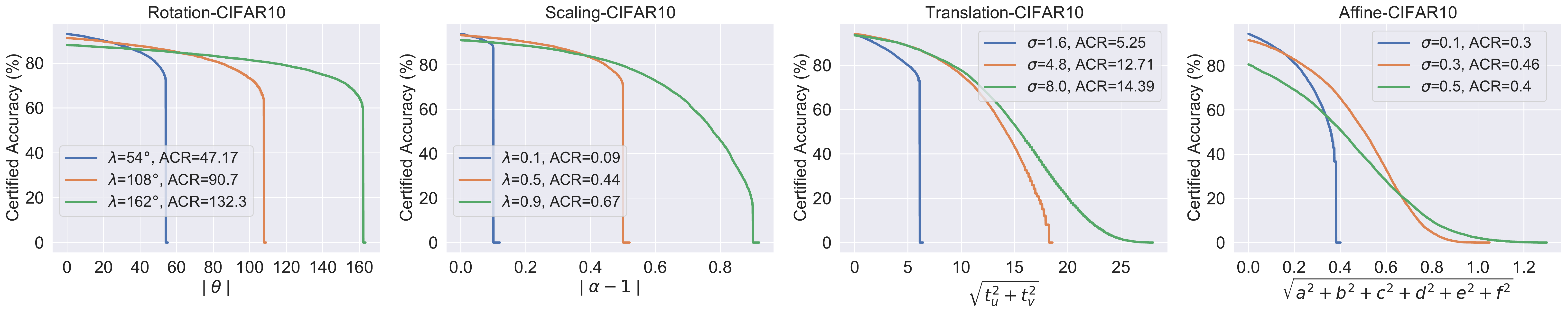}
    \includegraphics[width=0.85\textwidth]{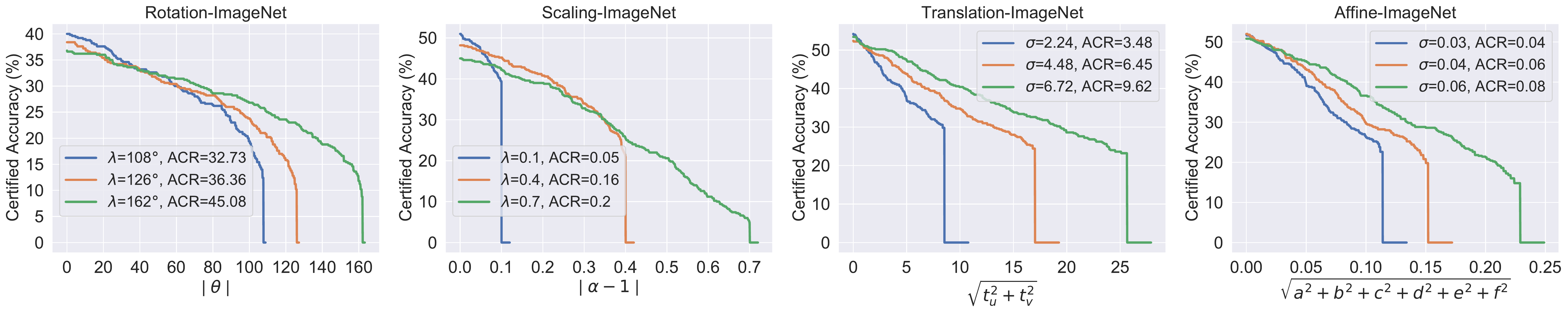}
    \caption{\small\textbf{Certified performance of \textsc{DeformRS-PAR}}. We show the effect of varying the smoothing parameters ($\lambda$, $\sigma$) on the certified accuracy of \textsc{DeformRS-PAR} against rotation, scaling, translation, and affine deformations.} \label{fig:LowDimTransformations_MNISTCIFARImagenet}
\end{figure*}

\textbf{Translation.}  Translation deformations are parameterized by two parameters $(t_u,t_v)$ that can generally be of any value. Thus, we employ two dimensional Gaussian smoothing as per Corollary \ref{cor:parametric_certification}, where $\sigma \in \{0.1, 0.2, \dots, 0.5\}$ for MNIST and CIFAR10, and $\sigma \in \{0.02, 0.03, \dots, 0.06\}$ for ImageNet. In this case, the granted certificate is of the form $(t_u^2 + t_v^2)^{\nicefrac{1}{2}} \leq \nicefrac{\sigma}{2} (\Phi^{-1}(p_A) - \Phi^{-1}(p_B))$. We compare against BAL, MOH, and LI and report the certified accuracy at a certification radius of at most 5 pixels.
\footnote{Since the image dimensions in our setting are normalized to $[-1, 1]$, we unnormalize the radius results to pixels in the original image for comparison and ease of interpretation in Figure \ref{fig:LowDimTransformations_MNISTCIFARImagenet}.}
As observed from Table \ref{tb:MNIST-CIFAR10}, \textsc{DeformRS-PAR} outperforms its best competitor by $2\%$ and $4\%$ on MNIST and CIFAR10, respectively. However, we observe that \textsc{DeformRS-Par} underperforms on ImageNet attaining $48.2\%$ certified accuracy compared to $63.3\%$ by LI as reported in Table \ref{tb:ImageNet}. We believe that \textsc{DeformRS-Par} performs worse on ImageNet due to the suboptimal training of the base classifier $f$ on ImageNet. This is evident in the third column of Figure \ref{fig:LowDimTransformations_MNISTCIFARImagenet}, which plots the certified accuracy over a range of different radii for several smoothing $\sigma$. Note that the certified accuracy of \textsc{DeformRS-Par} is $\sim 52\%$ at radius 0 over all $\sigma$. That is to say, the \textit{accuracy} of \textsc{DeformRS-Par} is already worse than the \textit{certified accuracy} at radius $5$ reported by LI.  However, the certified accuracy of \textsc{DeformRS-Par} for MNIST and CIFAR10  at radii of 7 and 8 pixels are at least 90\% and 80\% on MNIST and CIFAR10, respectively.

\begin{figure*}[t]
\centering
    \includegraphics[width=0.85\textwidth]{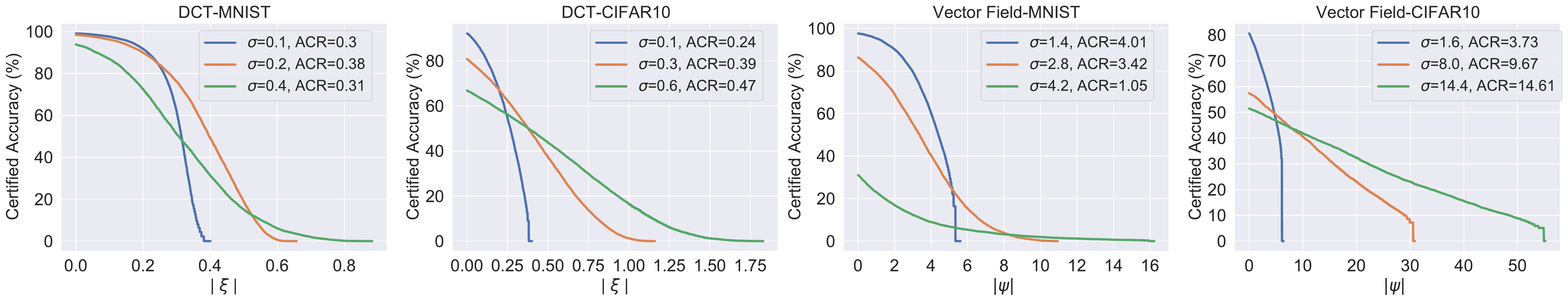}
    \caption{\small\textbf{Performance of \textsc{DeformRS-VF} and \textsc{DeformRS-PAR}.} We plot the certified accuracy curves of \textsc{DeformRS-PAR} against truncated DCT deformations (left) and \textsc{DeformRS-VF} against general vector field deformations (right).}\label{fig:HighDimTransformations}
\end{figure*}
\begin{figure}[t]
  \centering
\begin{tabular}{lll}
    \includegraphics[width=0.4\textwidth]{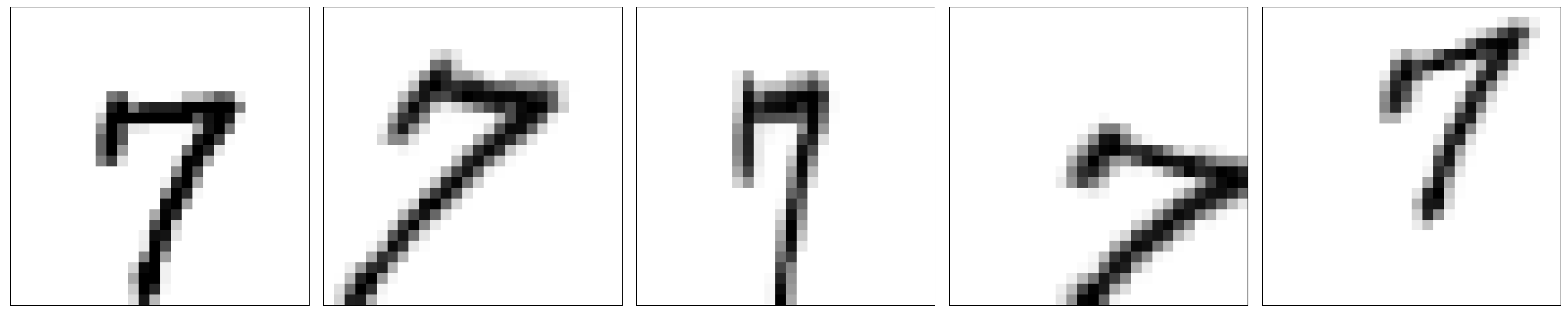} \\
    \includegraphics[width=0.4\textwidth]{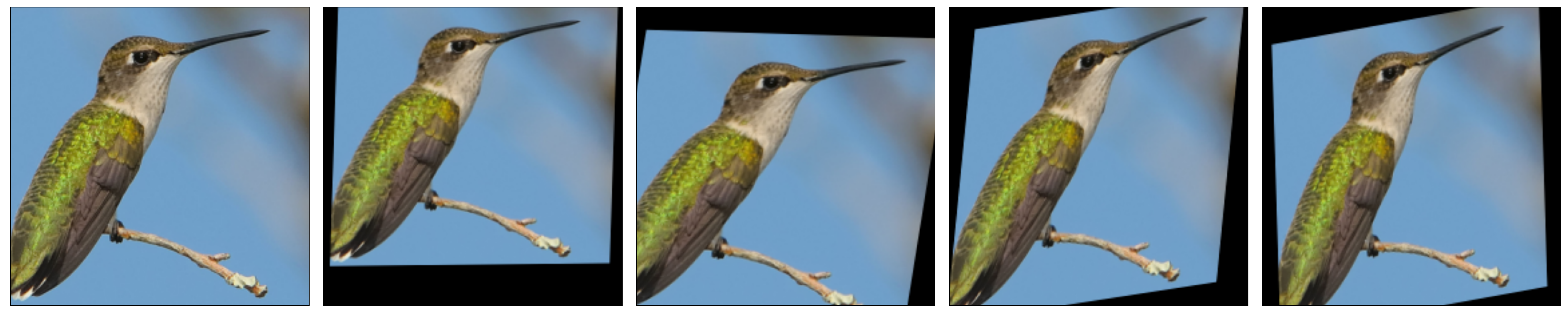}
\end{tabular}
    \caption{\small\textbf{Examples of certified affine deformations.} We sample affine parameters satisfying the certification inequality.
    }
    \label{fig:affine_samples}
\end{figure}

\subsection{\textsc{DeformRS-Par} against Affine Deformations}
Attaining high certified accuracy for individual deformations, as discussed earlier, requires the training of networks for these particular deformations. Thus, we train a \textit{single} \textsc{DeformRS-Par} network against affine deformations, where we certify it against several specializations of the affine certificate. Recall that the affine deformation is parameterized by 6 parameters. Since generally, there are no restrictions on the values of the affine parameters, we use Gaussian smoothing in Corollary \ref{cor:parametric_certification} to sample $\{a,b,c,d,e,f\}$ with $\sigma \in \{0.1, 0.2, \dots, 0.5\}$ for MNIST and CIFAR10 and $\sigma \in \{0.02, 0.03, \dots, 0.06\}$ on ImageNet. The certificate is thus in the form $\sqrt{a^2 + b^2 + c^2 + d^2 + e^2 + f^2} \leq \nicefrac{\sigma}{2} (\Phi^{-1}(p_A) - \Phi^{-1}(p_B))$. The last column of Figure \ref{fig:LowDimTransformations_MNISTCIFARImagenet} summarizes the certified accuracy of \textsc{DeformRS-PAR} on all three datasets. Note, the certified accuracy of \textsc{DeformRS-Par} at affine radius of $0.3$ on MNIST is $90\%$. This is equivalent, under specialization to a translation (\ie $a=b=c=d=0$), to a certified accuracy of $90\%$ for all translations of radius $0.15\times 28 = 4.2$ pixels (after unnormalization).

\textbf{Composition of deformations.} We specialize the certificate to a composition of several deformations and compare against the only work certifying deformation compositions (BAL). Following BAL, we consider the composition of shearing with a factor of $s$ followed by rotation with angle $\theta$. The vector field is given as follows:
\[
\begin{pmatrix}
u_{n,m}\\
v_{n,m}
\end{pmatrix}
= \underbrace{
\begin{pmatrix}
\cos(\theta)  & - \sin(\theta)\\
\sin(\theta)  & \,\,\,\,\, \cos(\theta) 
\end{pmatrix}
}_{\text{rotation}}
\underbrace{
\begin{pmatrix}
1  & s \\
0  & 1
\end{pmatrix}
}_{\text{shear}}
\begin{pmatrix}
n\\
m
\end{pmatrix}
-
\begin{pmatrix}
n\\
m
\end{pmatrix}.
\]

Note that this composition can be formulated as an affine deformation with $a =\cos(\theta) - 1$, $b = s \cos(\theta) - \sin(\theta)$, $c = \sin(\theta)$, $d = s \sin(\theta) + \cos(\theta) - 1$, and $e = f = 0 $. Therefore, Corollary \ref{cor:parametric_certification} grants the following certificate $\sqrt{s^2 -2s\sin(\theta)-4cos(\theta)+4} \leq \nicefrac{\sigma}{2}(\Phi^{-1}(p_A) - \Phi^{-1}(p_B))$. We compare against BAL, which achieves a certified accuracy of 54.2\% on CIFAR10 under the setting $|\theta| \leq 2\degree$ and $0 \leq s \leq 2\%$. To compute the certified accuracy of \textsc{DeformRS-Par}, note that the left hand side achieves its maximum of $0.0651$ at $\theta^*=-2\degree$ and $s^*=0.02$; thus, the certified accuracy of \textsc{DeformRS-Par} is the percentage of the test set classified correctly with a radius of at least $0.0651$. \textsc{DeformRS-Par} achieves a certified accuracy of 91.28\% on CIAFR10 and 43.6\% on ImageNet as per the last column in Figure \ref{fig:LowDimTransformations_MNISTCIFARImagenet}, thus outperforming BAL by 37\%. Note that, our affine certification allows for the seamless certification of all considered deformations in the literature. This surpasses any need to specialize a certificate for every deformation family of an affine nature. In fact, with  a single network trained with \textsc{DeformRS-Par} against affine deformations, we achieve non-trivial certified accuracies against several specialized deformations. Moreover, we consider certifying the same \textsc{DeformRS-Par} network under the composition of a rotation of angle $\theta$, scaling by factor $\alpha$, and a translation of parameters $(t_u,t_v)$. %
Under such a setting,
we have $a=\alpha \cos(\theta)-1$, $b=-\alpha\sin(\theta)$, $c= \alpha \sin(\theta)$, $d=\alpha\cos(\theta)-1$, $e=t_u$, and $f=t_v$. Therefore, this grants the following certificate $\sqrt{2 + 2\alpha^2 - 4\alpha \cos(\theta) + t_u^2 + t_v^2} \leq \nicefrac{\sigma}{2}(\Phi^{-1}(p_A) - \Phi^{-1}(p_B))$. We consider certifying \textsc{DeformRS-Par} under the composite deformation of $|\theta| \leq 10\degree$, $0.8 \leq \alpha \leq 1.2$, and $t_u^2 + t_v^2 \leq 0.1$, where $0.1$ corresponds to a radius of 4 and 5 pixels of translation for images in MNIST and CIFAR10, respectively. To that end, we observe that the left hand side of the certificate attains a maximum of $0.503$, at which \textsc{DeformRS-Par} enjoys a certified accuracy of 79.78\% on MNIST and 50.41\% on CIFAR10 as per the last column in Figure \ref{fig:LowDimTransformations_MNISTCIFARImagenet}. To the best of our knowledge, this work is the first to consider such a composite deformation. In Figure \ref{fig:affine_samples}, we sample several certifiable affine deformations that satisfy the certificate inequality and apply them to MNIST and ImageNet images. We can observe  the richness of the certifiable affine maps in both datasets.

\subsection{\textsc{DeformRS-Par} - Truncated DCT Deformations}

We go beyond affine deformations in this section to cover parameterized truncated DCT deformations; particularly, the class of vector field deformations generated by taking the inverse DCT transform of a truncated window of size $k \times k \times 2$. We observe that this class of deformations can generate visually aligned deformations, which are generally not affine, as shown in Figure \ref{fig:pull}. Since the $k \times k \times 2$ DCT coefficients can take any values, we use Gaussian smoothing with $\sigma \in \{0.1, 0.2, \dots, 0.5\}$  as per Corollary \ref{cor:parametric_certification}. This grants a certificate of the form $\|\xi\|_2 \leq \nicefrac{\sigma}{2}(\Phi^{-1}(p_A) - \Phi^{-1}(p_B))$, where $\xi$ is the perturbation in the DCT coefficients. For simplicity, we set $k=2$ for all the experiments in this section. As per Figure \ref{fig:HighDimTransformations}, \textsc{DeformRS-Par} certifies perturbations in the DCT coefficients with a certified accuracy of 90\% and 80\% at radius 0.2 on MNIST and CIFAR10, respectively. Unlike individual deformations or their compositions, it is generally difficult to interpret the certified class of DCT deformations; we instead visualise samples from the certified region of DCT coefficients in Figure \ref{fig:dct_vf_samples}. We observe interesting certified deformations that are visually aligned 
resembling different hand written digits in MNIST or ripples in CIFAR10.

\begin{figure}[t]
  \centering
\begin{tabular}{ll}
    \includegraphics[width=0.4\textwidth]{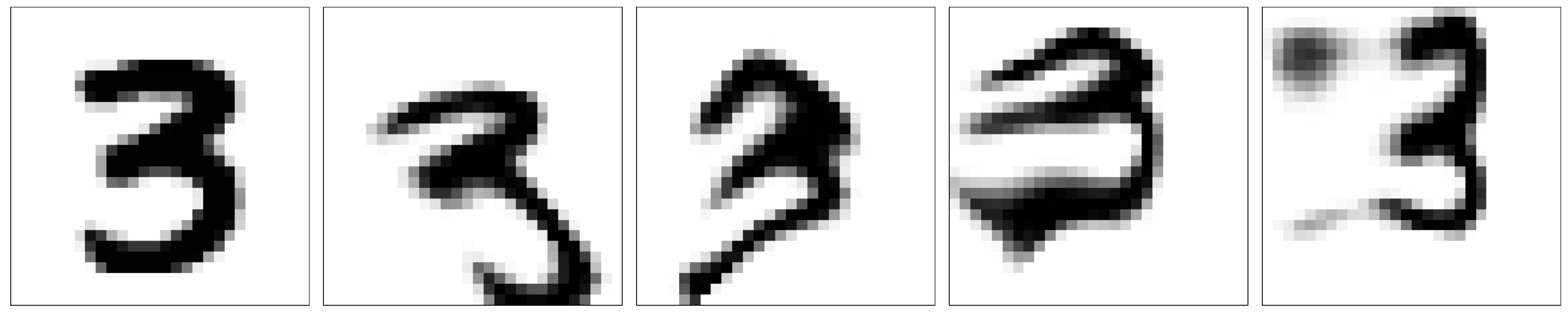}\\
    \includegraphics[width=0.4\textwidth]{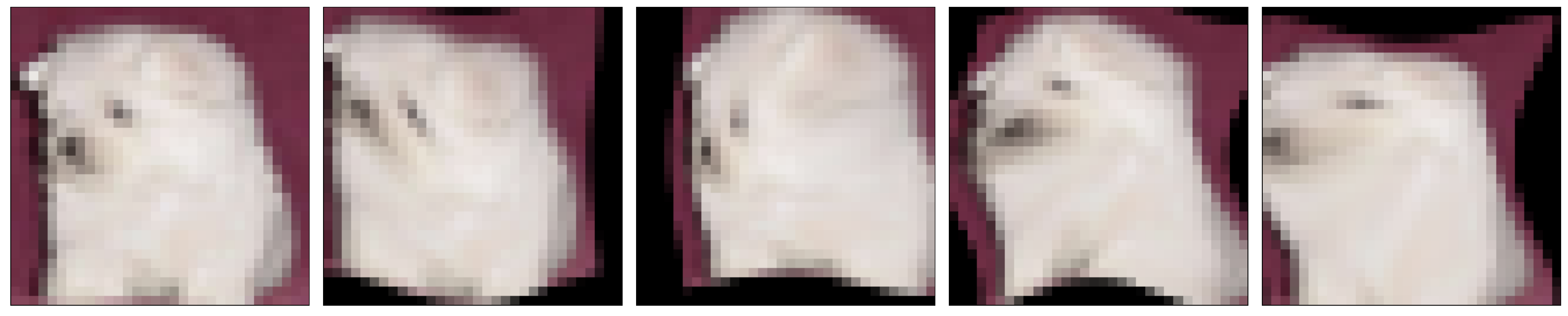}
\end{tabular}
    \caption{\small\textbf{Examples of certified truncated DCT.} We sample truncated DCT coefficients that satisfy the certification inequality.}
    \label{fig:dct_vf_samples}
\end{figure}
\subsection{\textsc{DeformRS-VF} - Vector Field Deformations}
We leverage Theorem \ref{theo:vector_field_certification} to certify against a general vector field deformation $\psi$. Note that such vector fields are in general of size $N \times M \times 2$ and can take any values. Thus, Gaussian smoothing is an appropriate choice, where we set $\sigma \in \{0.1, 0.2, \dots, 0.5\}$. We plot in the second row of Figure \ref{fig:HighDimTransformations} the certified accuracy of \textsc{DeformRS-VF} for an unnormalized vector field. We observe that \textsc{DeformRS-VF} achieves a certified accuracy of 90\% and 60\% at a radius of 2 pixels on MNIST and CIFAR10, respectively. Note that while vector field deformations can specialize to all previously considered deformations as special cases (\eg rotations), they suffer from the curse of dimensionality ($\psi$ is of size $2NM$) granting certification to only imperceptible deformations. For instance, consider the vector field generated from a parameterized translation such that $(t_u^2 + t_v^2)^{\nicefrac{1}{2}} \leq 2$. The corresponding vector field will have an energy of at most $\sqrt{2MN}$. That is to say, to certify vector field deformations representing translations of 2 pixels in $\ell_2$, the certification radius of the vector field should be at least $\sqrt{2MN}$, which is significantly larger than the radius $2$ with the earlier reported accuracy. This is a classical trade-off between the generality of the deformation family and the imperceptibility of the certifiable deformation. We leave the rest of the experiments for the \textbf{Appendix}.

\textbf{Acknowledgments.} This publication is based upon work supported by the King Abdullah University of Science and Technology (KAUST) Office of Sponsored Research (OSR) under Award No. OSR-CRG2019-4033.

\bibliography{references}
\input{Sections/appendices}

\end{document}

%% file: Sections/appendices.tex
\newpage
\onecolumn
\appendix
\section{Proofs}
We first start by defining Lipschitz continuity and the corresponding tightest Lipschitz constant.

\begin{proposition} Consider a differentiable function $g:\mathbb{R}^n \rightarrow \mathbb{R}$. If $\text{sup}_x \|\nabla g(x)\|_* \leq L$ where $\|\cdot\|_*$ has a dual norm 
$\|z\| = \max_x z^\top x ~~ ~\text{s.t. } \|x\|_* \leq 1$, then $g$ is $L$-Lipschitz under norm $\|\cdot\|_*$, that is $|g(x) - g(y)| \leq L \|x - y\|$.
\end{proposition}
\begin{proof}
Consider some $x, y \in \mathbb{R}^n$ and a parameterization in $t$ as $\gamma(t) = (1-t)x + ty ~~ \forall t\in[0,1]$. Note that $\gamma(0) = x$ and $\gamma(1) = y$. By the fundamental Theorem of calculus we have:
\begin{align*}
    |g(y) - g(x)| = \left|g(\gamma(1)) - g(\gamma(0))\right| &= \left|\int_{0}^1 \frac{dg(\gamma(t))}{dt} dt\right| = \left|\int_{0}^1 \nabla g^\top \nabla \gamma dt\right| \leq  \int_{0}^1 \left|\nabla g^\top \nabla \gamma \right| dt\\
    & \leq \int_{0}^1 \|\nabla g(x)\|_* \|\nabla \gamma(t)\| dt \leq L \|y-x\| 
\end{align*}
\end{proof}

\noindent To prove Theorem \ref{theo:vector_field_certification} and corollary \ref{cor:parametric_certification}, we note that any Lipschitz function is certifiable.

\begin{theorem}\label{theo:smoothness-certificate}
Let $f:\mathbb{R}^n \rightarrow \mathbb{R}$, $f^i$ be $L$-Lipschitz continuous under norm $\|\cdot\|_*$ $\forall i\in\{1,\dots,K\}$, and $c_A = \text{arg}\max_i f^i(x)$. Then, we have $\text{arg}\max_i f^i(x+\delta) = c_A$ for all $\delta$ satisfying:
\[\|\delta\| \leq \frac{1}{2L}\left(f^{c_A}(x)- \max_c f^{c \neq c_A}(x)\right).
\]
\end{theorem}
\begin{proof}
Take $c_B = \argmax_c f^{c \neq c_A}(x)$. Hence:
\begin{align*}
    |f^{c_A}(x+\delta) - f^{c_A}(x) |\leq L\|\delta\| \implies f^{c_A}(x+\delta) \geq f^{c_A}(x) - L\|\delta\|\\
    |f^{c_B}(x+\delta) - f^{c_B}(x) |\leq L\|\delta\| \implies f^{c_B}(x+\delta) \leq f^{c_B}(x) + L\|\delta\|
\end{align*}
By subtracting the inequalities and re-arranging terms, we have that as long as $f^{c_A}(x) - L\|\delta\| > f^{c_B}(x) + L\|\delta\|$, \ie the bound in the Theorem, then $f^{c_A}(x+\delta) > f^{c_B}(x+\delta)$, completing the proof.
\end{proof}

\noindent Next, we deploy Theorem \ref{theo:smoothness-certificate} to prove Theorem \ref{theo:vector_field_certification} and corollary \ref{cor:parametric_certification}. We first revisit the definition of a deformation smoothed classifier.

\begin{definition}[restatement]\label{defi:smooth_domain_classifier}
Given a classifier $f : \mathbb R^n \rightarrow \mathcal P(\mathcal Y)$ and an interpolation function $I_T : [0,1]^n \times \mathbb{R}^{2|\Omega^\mathbb{Z}|} \rightarrow [0,1]^n$, we define a deformation smoothed classifier as:
\begin{align*}
    \hat{g}(x,p) = \mathbb{E}_{\epsilon \sim \mathcal{D}}
    \left[f\left(I_T(x, p + \epsilon)\right)\right].
\end{align*}
\end{definition}

At first, following Theorem \ref{theo:smoothness-certificate}, we show that $\hat{g}(x,p)$ is Lipschitz in $p$.

\begin{proposition}
$\hat{g}(x,p) = \mathbb{E}_{\epsilon \sim \mathcal{U}[-\lambda,\lambda]}
    \left[f\left(I_T(x, p + \epsilon)\right)\right]$ is $\nicefrac{1}{2\lambda}-$Lipschitz in $p$ under $\|\cdot\|_\infty$ norm.
\end{proposition}  
\begin{proof}
It suffices to show that $\|\nabla_p \hat{g}(x,p)\|_\infty \leq \nicefrac{1}{2\lambda}$ to complete the proof. Without loss of generality, we analyze $\nicefrac{\partial \hat g}{\partial p_1}$. Let $\hat p = [p_2, \dots, p_n] \in \mathbb{R}^{n-1}$, then:
\begin{align*}
    \frac{ \partial\hat g}{\partial p_1} &= \frac{1}{(2\lambda)^n}\frac{\partial}{\partial p_1} \int_{[-\lambda,\lambda]^{n-1}} \int_{-\lambda}^{\lambda} f(I_T(x, p_1+ \epsilon_1, \hat{p} +  \hat{\epsilon})) d\epsilon_1 d^{n-1}\hat \epsilon \\ 
    &= \frac{1}{(2\lambda)^n} \int_{[-\lambda,\lambda]^{n-1}} \frac{\partial}{\partial p_1} \int_{p_1-\lambda}^{p_1+\lambda} f(I_T(x, t, \hat{p} +  \hat{\epsilon})) dt d^{n-1}\hat \epsilon \\ 
    &= \frac{1}{(2\lambda)^n} \int_{[-\lambda,\lambda]^{n-1}}  f(I_T(x, p_1+\lambda, \hat{p} +  \hat{\epsilon})) - f(I_T(x, p_1\lambda, \hat{p} +  \hat{\epsilon})) dt d^{n-1}\hat \epsilon 
\end{align*}
Thus, 
\begin{align*}
    \left|\frac{ \partial\hat g}{\partial p_1}\right| &\leq \frac{1}{(2\lambda)^n} \int_{[-\lambda,\lambda]^{n-1}}  \left|f(I_T(x, p_1+\lambda, \hat{p} +  \hat{\epsilon})) - f(I_T(x, p_1\lambda, \hat{p} +  \hat{\epsilon}))\right| dt d^{n-1}\hat \epsilon  \leq \frac{1}{2\lambda}.
\end{align*}
The second and last steps follow with the change of variable $t = p_1+ \epsilon_1 \epsilon_1$ and Leibniz rule. Similarly, $\left| \nicefrac{\partial\hat g}{\partial p_i}\right| \leq \nicefrac{1}{2\lambda} \,\, \forall i$. This results in having $\|\nabla_p \hat g(x)\|_{\infty} = \max_i  \left|\nicefrac{\partial\hat g}{\partial p_i}\right| \leq \nicefrac{1}{2\lambda} $.
\end{proof}

At last, we show that the composite $\Phi^{-1}(\hat{g}(x,p)$ is Lipschitz in $p$ under $\|\cdot\|_2$ norm. Following Theorem \ref{theo:smoothness-certificate}, we show that $\hat{g}(x,p)$ is Lipschitz in $p$.

\begin{proposition}
$\Phi^{-1}(\hat{g}(x,p)) = \Phi^{-1}(\mathbb{E}_{\epsilon \sim \mathcal{N}(0,\sigma^2I)}
    \left[f\left(I_T(x, p + \epsilon)\right)\right])$ is $\nicefrac{1}{\sigma}-$Lipschitz in $p$ under $\|\cdot\|_2$ norm.
\end{proposition}

\begin{proof}
We want to show that $\|\nabla \Phi^{-1}(\hat{g}(x,p))\|_2 \leq 1$. Following the argument presented in~\cite{salman2019provably}, it suffices to show that, for any unit norm direction $u$ and $r = \hat{g}(x,p)$, we have:
\begin{equation}
    \label{eq:lipschitz_bound_l2}
    \sigma u^\top \nabla_p \hat{g}(x,p) \leq \frac{1}{\sqrt{2\pi}}\exp\left(-\frac{1}{2} (\Phi^{-1}(r))^2\right).
\end{equation}
We start by noticing that:
\begin{align*}
  \sigma u^\top \nabla_p \hat{g}(x,p)
    & = \mathbb E_{v\sim \mathcal N (0,  \sigma^2 I)} \left[f(I_T(x, p+v)) \frac{u^\top v}{\sigma}\right].
\end{align*}

\noindent We now need to find the optimal $f^*:\mathbb{R}^n\rightarrow [0,1]$ that satisfies $\hat{g}(x,p) = \mathbb{E}_{v \sim \mathcal{N}(0,\sigma^2I)}[f^*(I_T(x,v+p))] = r$ while maximizing $\mathbb E_{v\sim \mathcal N (0,  \sigma^2 I)} [f(I_T(x, p+v)) \nicefrac{u^\top v}{\sigma}]$.We argue that the maximizer is the following function:
\[
f^*(I_T(x,p+v)) = \mathbbm{1}\left\{\frac{u^\top v}{\sigma} \ge - \Phi^{-1}(r)\right\}.
\]
To prove that $f^*$ is indeed the maximizer, we first show feasibility. \textbf{(i)}: 
It is clear that $f^*: \mathbb{R}^n \rightarrow [0,1]$, and note that (\textbf{ii}):
\begin{align*}
    \mathbb{E}_{v \sim \mathcal{N}(0,\sigma^2I))}\left[\mathbbm{1}\left\{\frac{u^\top v}{\sigma} \ge - \Phi^{-1}(r)\right\}\right] &= \mathbb{P}_{x \sim \mathcal{N}(0,1)}(x \ge -\Phi^{-1}(r)) = 1 - \Phi(-\Phi^{-1}(r)) = r.
\end{align*}
To show the optimality of $f^*$, we show that it attains the upper bound:
\begin{align*}
    \mathbb{E}_{v\sim \mathcal{N}(0,\sigma^2I))} \Big[\mathbbm{1}\left\{\frac{u^\top v}{\sigma} \ge-\Phi^{-1}(r)\right\} \frac{u^\top  v}{\sigma}\Big]  &= \mathbb{E}_{x \sim \mathcal{N}(0,1)}\Big[x \mathbbm{1}\left\{x \ge -\Phi^{-1}(r)\right\}\Big] \\
    & = \frac{1}{\sqrt{2\pi}}\int_{-\Phi^{-1}(r)}^\infty x \exp\left(-\frac{1}{2}x^2\right)dx \\
    & = \frac{1}{\sqrt{2\pi}} \exp\left(-\frac{1}{2} (\Phi^{-1}(r))^2\right).
\end{align*}
\end{proof}

\begin{theorem}\label{theo:vector_field}[restatement] Consider $\hat{g}(x,p) = \mathbb{E}_{\epsilon \sim \mathcal{D}}
    \left[f\left(I_T(x, p + \epsilon)\right)\right]$ where $\hat g$ assigns the class $c_A$ for an input $x$, \ie $c_A = \argmax_c \hat{g}^c(x,p)$ with:
\[
p_A = \hat{g}^{c_A}(x,p) \quad\text{ and }\quad p_B = \max_{i \neq c_A} \hat{g}^{i}(x,p)
\]
then $\argmax_c \hat g^c(x, p+\psi) = c_A$ for vector field perturbations satisfying:
\begin{equation}
\begin{aligned}
\begin{split}
    &\|\psi\|_1 \leq \lambda \left(p_A - p_B\right) \qquad \qquad \qquad \quad \,\,\,\, \text{for } \mathcal D = \mathcal U[-\lambda, \lambda],  \\
    &\|\psi\|_2 \leq \frac{\sigma}{2}\left(\Phi^{-1}(p_A) - \Phi^{-1}(p_B)\right) \qquad \text{for } \mathcal D = \mathcal N(0, \sigma^2I),
\end{split}
\end{aligned}
\end{equation}
\end{theorem}
\begin{proof}
The proof follows immediately by substituting $L$ and $f(x)$ in Theorem \ref{theo:smoothness-certificate} by  $L=\nicefrac{1}{2\lambda}$, $\hat{g}(x,p)$ and $L = \nicefrac{1}{\sigma}$, $\Phi^{-1}(\hat{g}(x,p))$ for Uniform and Gaussian distributions, respectively. Due to the monotonicity of $\Phi^{-1}$, $\max_{i\neq c_A} \Phi^{-1}(\hat{g}^i(x,p)) =  \Phi^{-1}(\max_{i\neq c_A} \hat{g}^i(x,p))$ completing the proof.
\end{proof}

\begin{corollary}\label{corr:param_deform}[restatement]
Suppose that $\tilde g$ assigns the class $c_A$ for an input $x$, \ie $c_A = \text{arg} \max_c \tilde{g}_\phi(x,p)$ with:
\[
p_A = \tilde{g}_\phi^{c_A}(x,p) \quad\text{ and }\quad p_B = \max_{i \neq c_A} \tilde{g}_\phi^{i}(x,p),
\]
then $\text{arg}\max_c \tilde g_{\phi+\xi}(x, p) = c_A$ for all parametric domain perturbations satisfying:
\begin{align}
\begin{split}
    &\|\xi\|_1 \leq \lambda \left(p_A - p_B\right) \qquad \qquad \qquad \quad \,\,\,\, \text{for } \mathcal D = \mathcal{U}[-\lambda, \lambda], \\
    &\|\xi\|_2 \leq \frac{\sigma}{2}\left(\Phi^{-1}(p_A) - \Phi^{-1}(p_B)\right) \qquad \text{for } \mathcal D = \mathcal N(0, \sigma^2I).
\end{split}
\end{align}
\end{corollary}
\begin{proof}
 The proof is identical to the proof of Theorem \ref{theo:vector_field_certification} but with replacing $\hat g$ with $\tilde g$, and $p$ with $\phi$. 
\end{proof}

\newpage
We provide below the general intuition behind Theorem \ref{theo:vector_field_certification} and how it is related to Theorem \ref{theo:vector_field} and Corollary \ref{corr:param_deform}.

\begin{itemize}
    \item[\textbf{(i)}] Classifiers that are Lipschitz are certifiable with a closed formula for their certification radius [see Theorem 2 in the Appendix].
    
    \item[\textbf{(ii)}] Randomized smoothing turns any classifier to a Lipschitz classifier with respect to the smoothed parameters (in our case either the vector field deformation, or the parameters of the deformation function). 
    
    \item[\textbf{(iii)}] \textbf{(i)} + \textbf{(ii)} $\implies$ a smooth classifier is certifiable against the perturbations in the smoothed parameters with a closed formula for the certification radius. 
    
    \item[\textbf{(iv)}] The choice of the smoothing distribution (Gaussian or Uniform in our case) only plays a role in deciding the norm of the certificate. Finally, the parameters of the smoothing distribution decides the value of the Lipschitz constant.
\end{itemize}

\section{Few Comments about the DCT Deformations}
Discrete Cosine Transforms (DCT) expresses an input signal as a finite sum of cosines at different frequencies. We consider a single dimensional input $T \in \mathbb{R}^N$ that denotes a transformation of a 1D dimensional signal of size $N$ where the extension to the $2D$ follows immediately.  The DCT of $T$ is given as follows:
    $\bar{T}_k = \sum_{n=0}^{N-1} a_n T_{n} \cos\left(\frac{\pi k}{N}\left(n + \frac{1}{2}\right)\right)$, 
where $k \in \{0, 1, \dots, N-1\}$. Similarly to Discrete Fourier Transform, DCT is a linear operator, therefore it can be expressed in a matrix-vector multiplication as $ \bar{T} = \mathbf{C}T$ where $\mathbf{C} \in \mathbb{R}^{N \times N}$ is the discrete cosine matrix such that $\mathbf{C}_{i,j} = a_i \cos\left(\frac{\pi (i-1) }{2}\left((j-1) + \frac{1}{2}\right)\right)$. Note that the number of basis of this transform is $N$, which is similar to the input dimension. We consider certifying the truncated version of the DCT transformation where we have $\tilde N \ll N$ basis, \ie $a_n=0 ~\forall n \ge \tilde{N}$. That is to say, a DCT certification of radius $r$ as per Section \textcolor{red}{3.2} has the following form $\|\xi\|_2 = \sqrt{\sum_i^{\tilde{N}} a_i^2} \leq r$. The smooth classifier \textsc{DeformRS-Par} is constructed by generating deformations using truncated DCT transformation and hence certifies the parameters of $\tilde{N}$ basis. Throughout our experiments, we set $\tilde{N} = 2$; we refer to $\tilde{N}$ in the main manuscript as $k \times k$ with $k=2$ representing the 2D DCT setting.

\begin{figure}[t]
    \centering
    \begin{tabular}{ll}
    \includegraphics[width=0.5\textwidth]{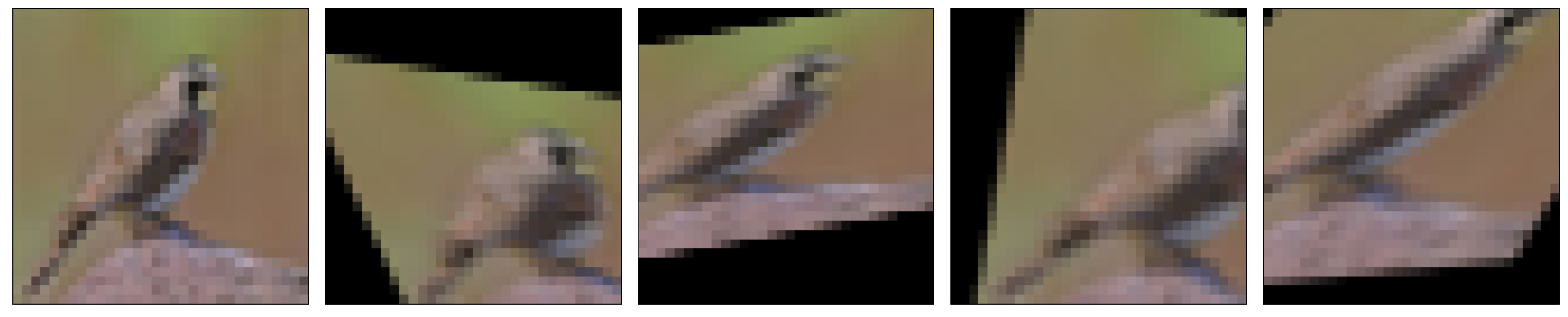}
    \includegraphics[width=0.5\textwidth]{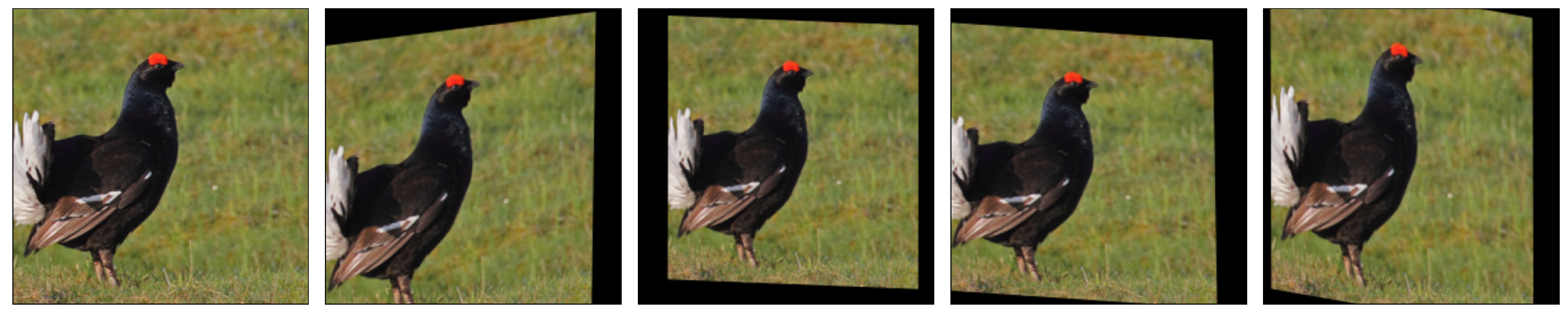}\\
    \includegraphics[width=0.5\textwidth]{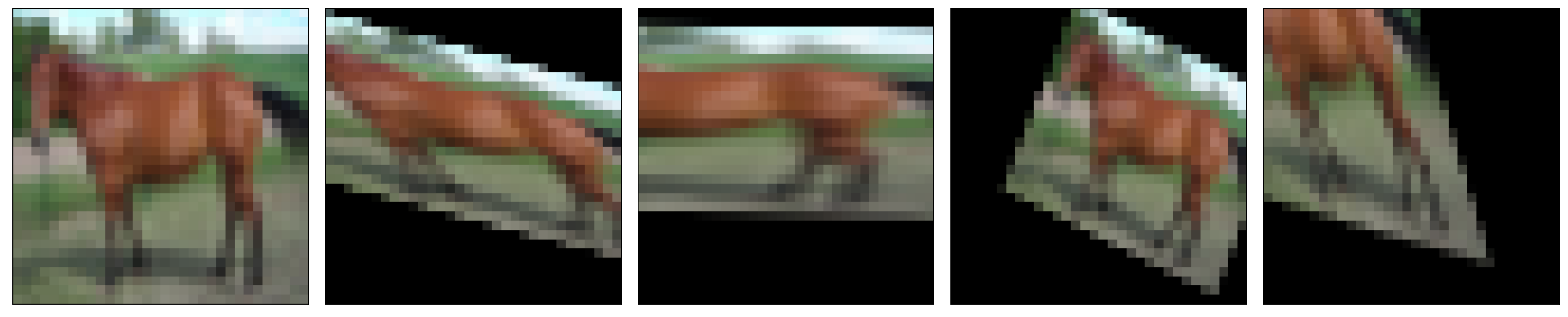}
    \includegraphics[width=0.5\textwidth]{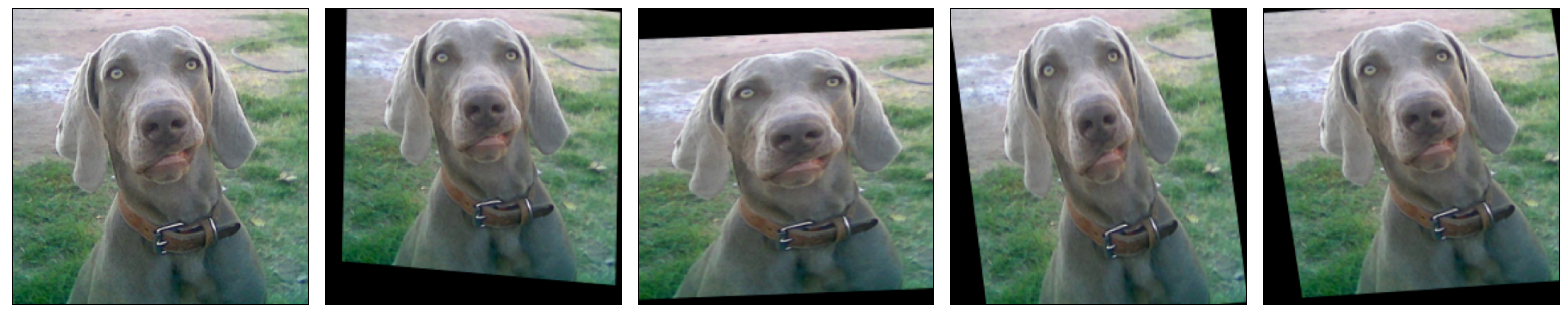}
    \end{tabular}
    \caption{\textbf{Certified examples against affine deformations.} The Figure shows 4 ImageNet images and several examples of affine deformations that are certifiable, i.e. \textsc{DeformRS-Par} produces a correct prediction under all the deformed examples.}
    \label{fig:appendix-affine}
\end{figure}

\section{Limitations, Broader Impact, and Compute Power}
\paragraph{Limitations.} The main limitation of this work, and any certification work, is the running time to compute the prediction of the classifier, and the certified radius around a given input point. However and since we deployed randomized smoothing, this affordable extra computation comes with the benefit of providing a network that is not only accurate, but certifyable robust against a variety of input deformations. Moreover, as stated earlier, the nature of threat model of consideration here is a threat model that can only alter parameters of a certain deformation class. This is similar in spirit to \cite{cvpr2002verification} While this is a more relaxed version of certification compared to methods that certify whether the perturbed inputs under a specific deformation produce similar predictions, this class of certification is of significant interest in the presence of simulators and generative models where adversaries have access to parameters of the deformation model.

\paragraph{Broader Impact.} Deep Neural Networks (DNNs) have achieved state-of-the-art results in a variety of computer vision tasks. However, this impressive performance was shown to be brittle against imperceptible input variations.  These variations could be not only input perturbations (varying pixel intensities) but also geometric transformations (\eg rotating the input image). Since these variations are likely to happen in real world scenarios, as the camera might experience some accidental movements, the deployment of DNNs in safety-critical applications (\eg self driving cars) is limited. This work takes a step into solving this issue by building classifiers that are certifiably robust against a variety of input deformations and hence increase the reliability of DNNs.

\paragraph{Compute Power.} In all of our training experiments, we used a single NVIDIA 1080-TI for CIFAR10 and MNIST experiments while we used 2 NVIDIA V100 to fine tune ImageNet models. For the certification experiments, we use a single gpu per experiments (  NVIDIA 1080-TI for CIFAR10 and MNIST and NVIDIA V100 for ImageNet).

\section{Visualizations}
We provide more samples that are within the certifiable radius for affine, DCT, and VF deformations. For affine, we show in Figure \ref{fig:appendix-affine} 2 examples from CIFAR10 (left) and ImageNet (right). For DCT deformation, we show in Figure \ref{fig:appendix-dct} 4 examples from both MNIST (right) and CIFAR19 (left). Note that all samples are classified correctly, as the norm of their DCT parameters is less than the certified radius. Moreover, it is to observe that DCT deformations provide semantically meaningfull deformations presented as ripples and stretches. Last, we provide 8 examples of certifiable full VF deformations in Figure \ref{fig:appendix-vf}. Note that since the certified norm is small, some of these deformations are imperceptible (specially on MNIST). We also noticed that some samples in CIFAR10 obtained large certified radius, hence could resist perceptible deformations as shown in the last row. We believe that this is due to the fact that rgb images are easier to discriminate not only based on texture, but also based on colors.




\section{Qualitative and Time Comparisons}

\textbf{Regarding the qualitative comparison,} we show examples in Figure \ref{fig:qualititave} of input samples that can be certified by DeformRS while are not certifiable and mispredicted by the models of LI \cite{li2020provable}.

\noindent \textbf{Regarding the time comparison,} we show in Table \ref{tb:time} the certification run time in seconds on CIFAR10 over all considered deformations. Moreover, we report the run time of LI. We compare the certification run time on 100 samples that are correctly classified by both DeformRS and LI. Our approach significantly outperforms LI in certification run time.

\begin{table}[h]
\small
    \centering
    \caption{We report certification time comparing our proposed DeformRS against LI \cite{li2020provable} on several classes of defomrations, i.e. translation (T), Rotation (R), Scaling (S), DCT, Affine and Vector Fields (VF).}
\begin{tabular}{cc|c|c|c|c|c}
 \toprule
\multicolumn{7}{c}{Certification Run Time in Seconds}
\\
 \midrule
& T& R & S &DCT & Affine &VF\\
 \midrule
  DeformRS & 8.70 & 12.37 & 12.197 &15.97 &9.83 &9.06\\
 \hline
  \cite{li2020provable} & 27.20 &461.71   & 942.88 & - & - & -\\
 \bottomrule
 \label{tb:time}
\end{tabular}
\vspace{-0.5cm}
\end{table}

\begin{figure}[h]
    \centering
    \includegraphics[width=0.75\textwidth]{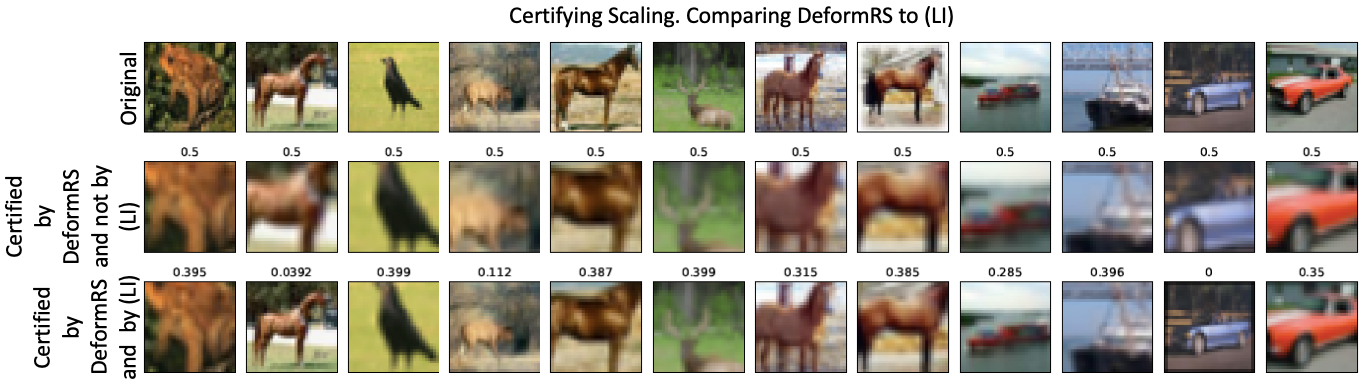}
    \caption{\textbf{Qualitative comparison.} Here we show examples that DeformRS can certify while fooling LI.}
    \label{fig:qualititave}
\end{figure}

\begin{figure}[H]
    \centering
    \begin{tabular}{ll}
    \includegraphics[width=0.35\textwidth]{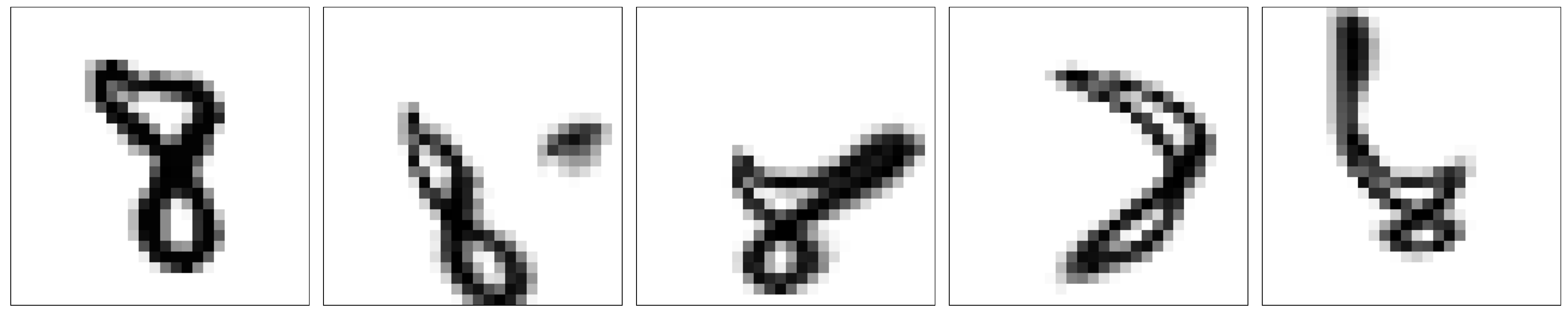}
    \includegraphics[width=0.35\textwidth]{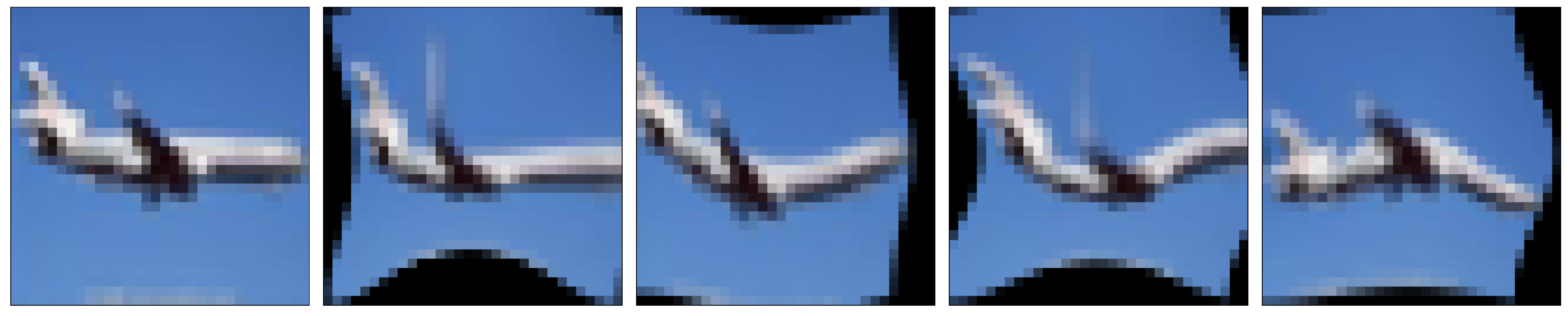} \\
    \includegraphics[width=0.35\textwidth]{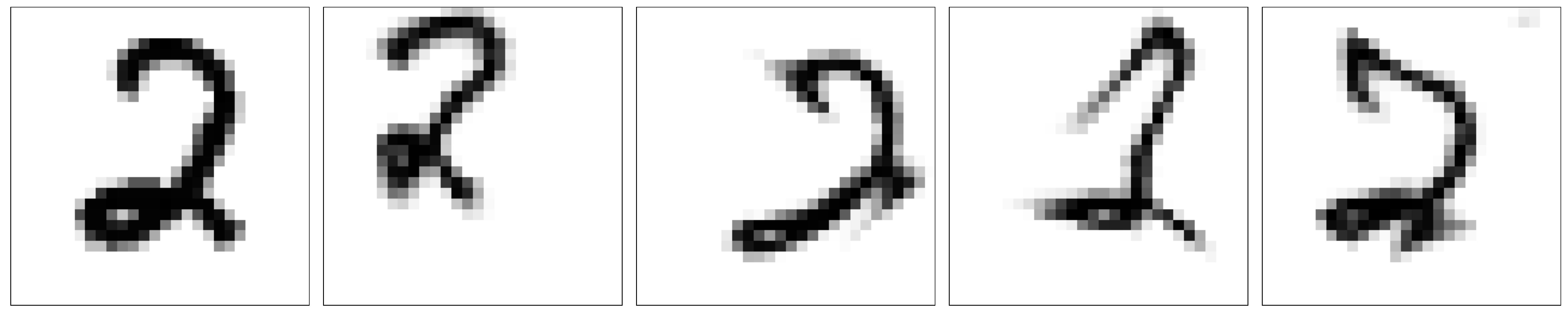}
    \includegraphics[width=0.35\textwidth]{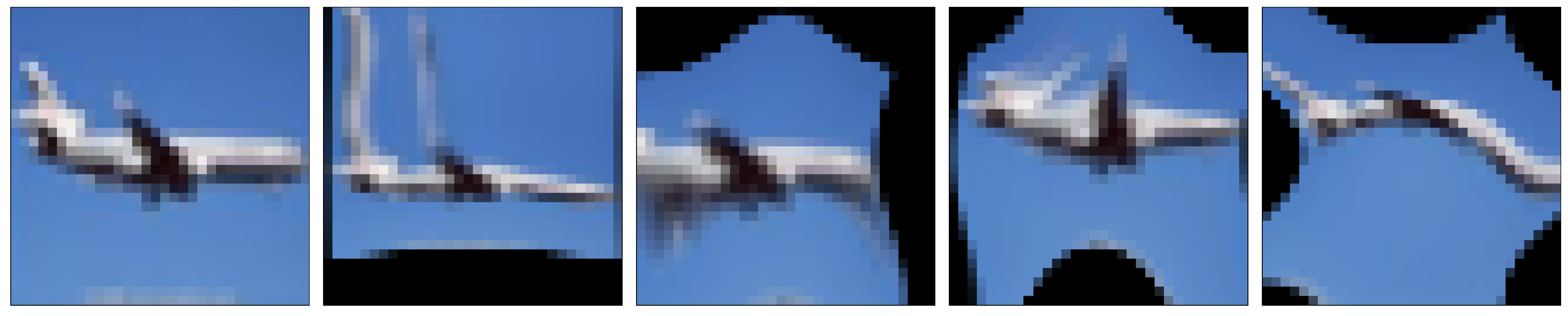} \\
    \includegraphics[width=0.35\textwidth]{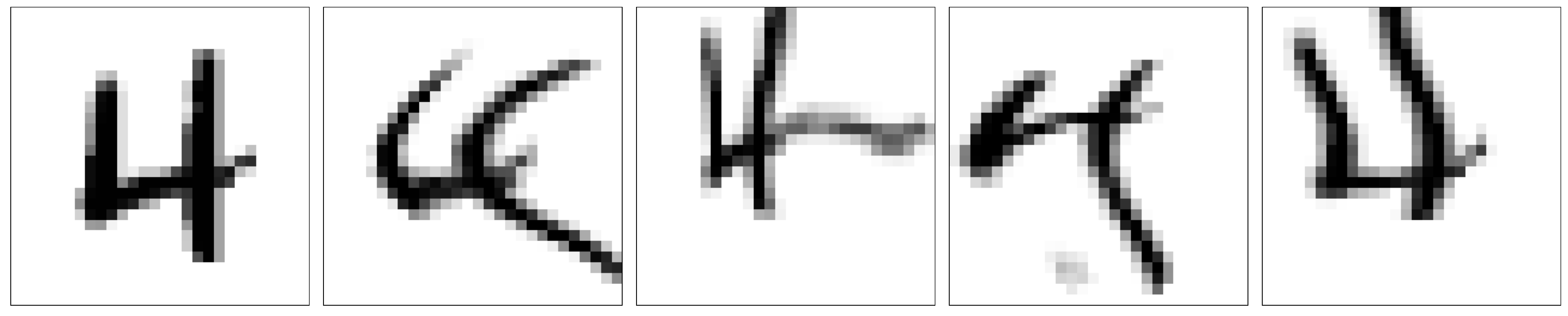}
    \includegraphics[width=0.35\textwidth]{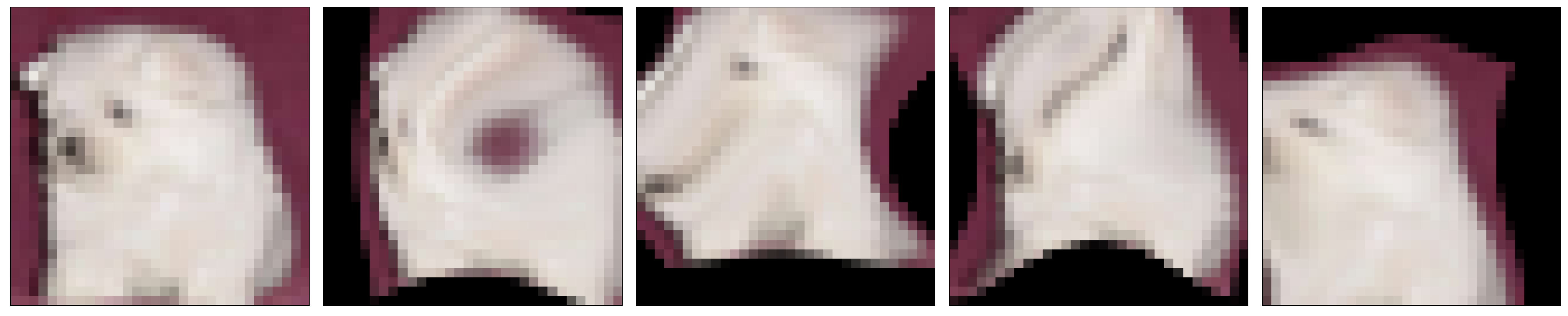} \\
    \includegraphics[width=0.35\textwidth]{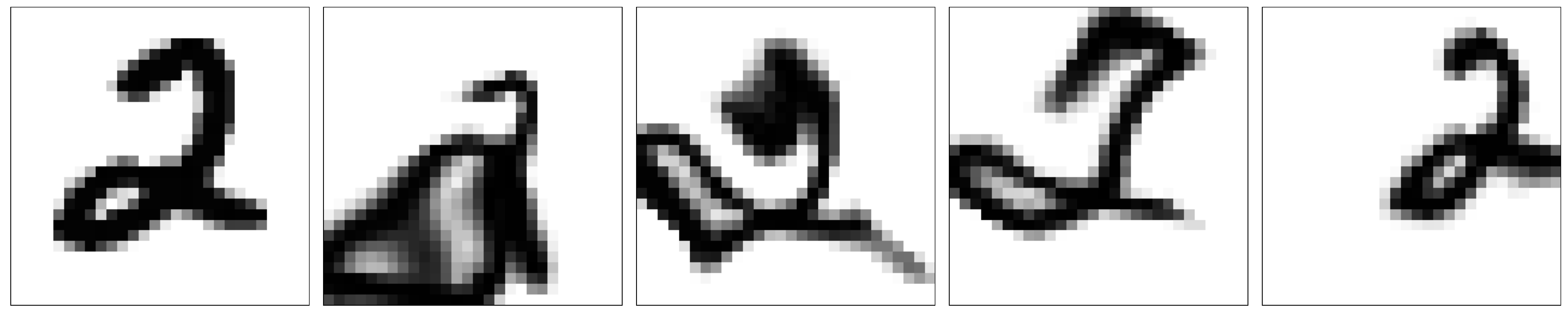}
    \includegraphics[width=0.35\textwidth]{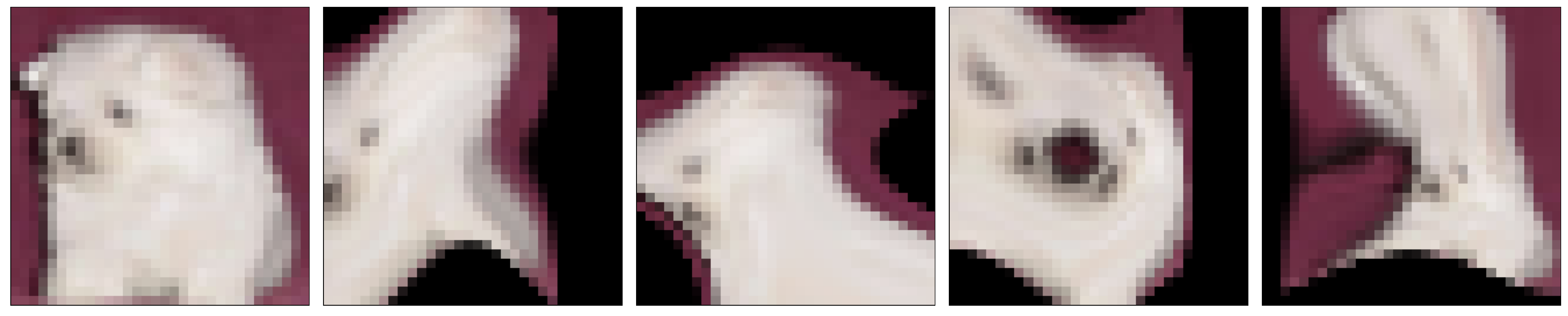}
    \end{tabular}
    \caption{\textbf{Certified examples against DCT deformations.} Figure shows 2 MNIST and 2 CIFAR10 images and several examples of DCT deformations that are certifiable, i.e. \textsc{DeformRS-Par} produces a correct prediction under all the deformed examples.}
    \label{fig:appendix-dct}
\end{figure}

\begin{figure}[t]
    \centering
    \begin{tabular}{ll}
    \includegraphics[width=0.35\textwidth]{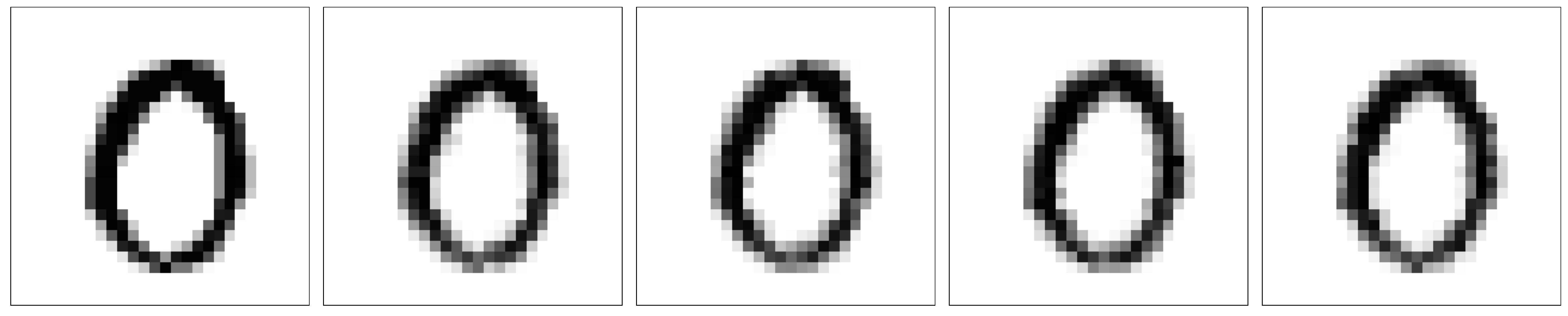}
    \includegraphics[width=0.35\textwidth]{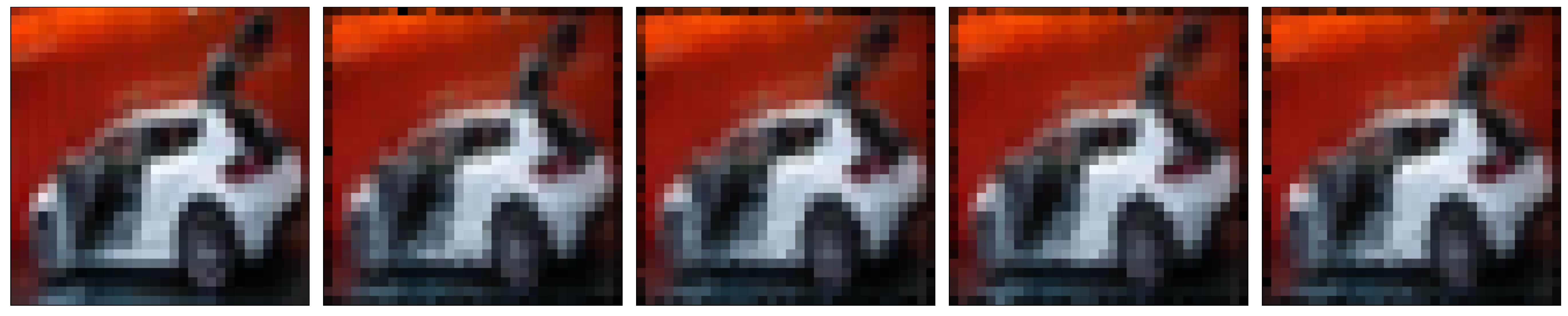} \\
    \includegraphics[width=0.35\textwidth]{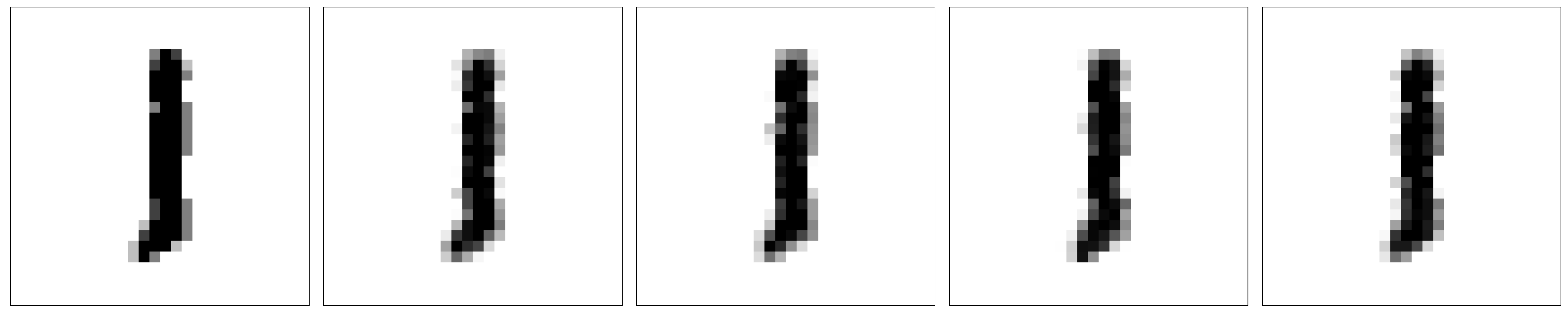}
    \includegraphics[width=0.35\textwidth]{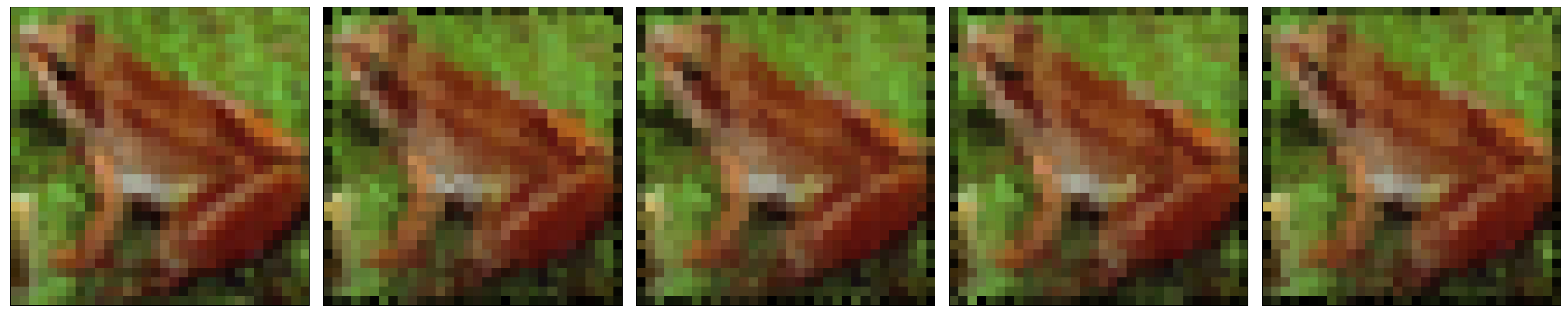} \\
    \includegraphics[width=0.35\textwidth]{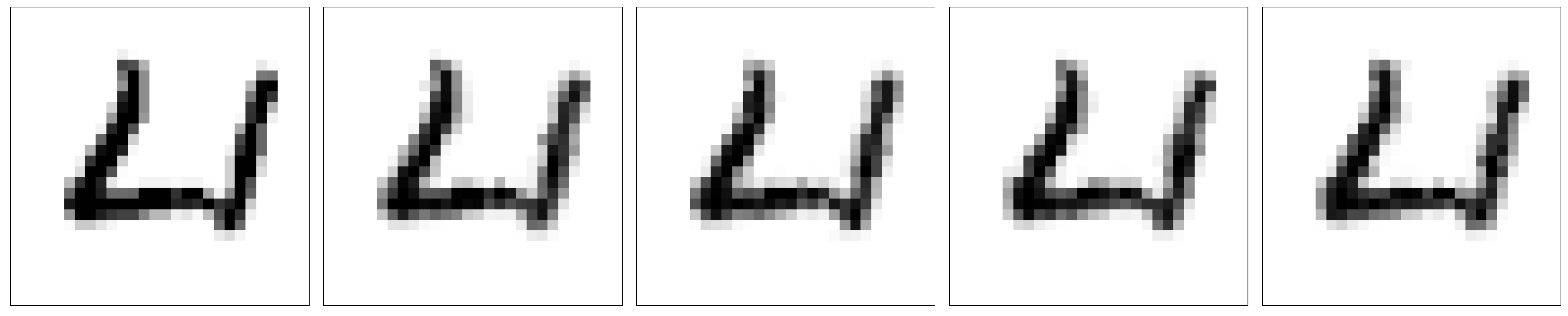}
    \includegraphics[width=0.35\textwidth]{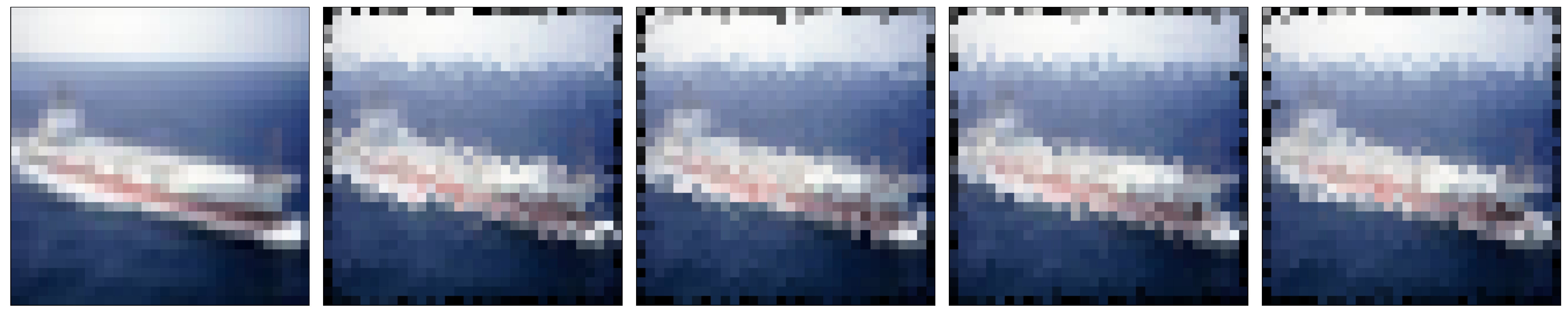}\\
    \includegraphics[width=0.35\textwidth]{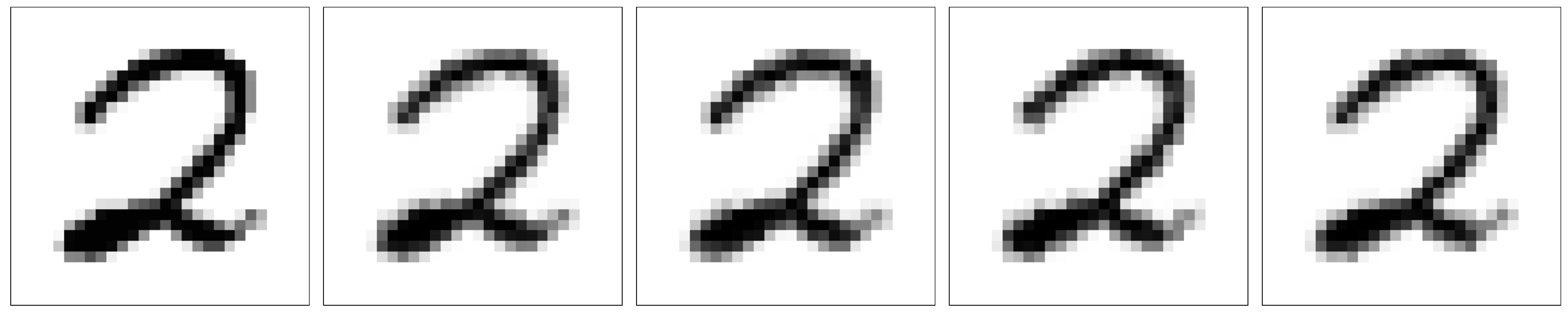}
   \includegraphics[width=0.35\textwidth]{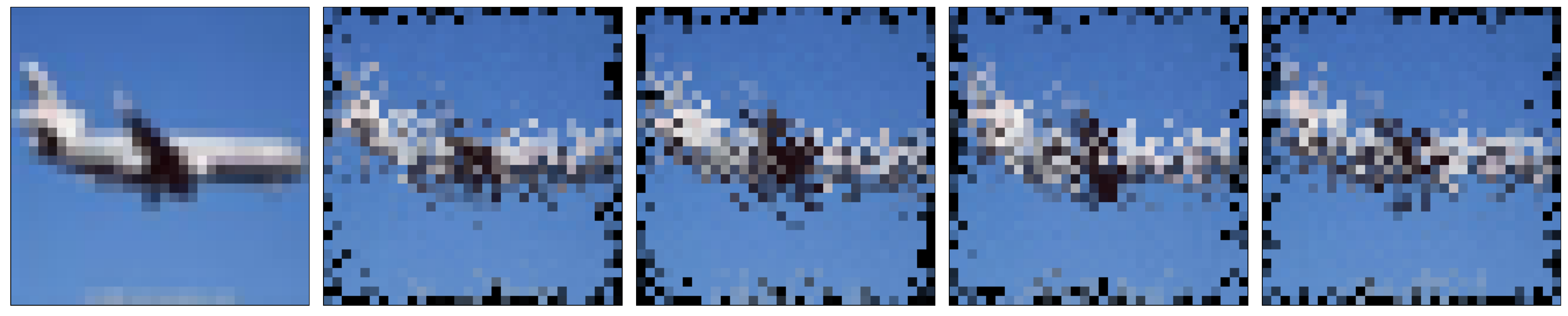}
    \end{tabular}
    \caption{\textbf{Certified examples against vector field deformations.} The Figure shows 2 MNIST and 2 CIFAR10 images along with several vector field deformed images that are certifiable, i.e. \textsc{DeformRS-Par} produces a correct prediction under all the deformed examples.}
    \label{fig:appendix-vf}
\end{figure}

\section{Ablations}
We provide here an ablation study to the effect of the parameters of the smoothing distribution to the performance of our smooth classifier. Not that as expected, the larger the smoothing parameters (\ie $\lambda$ and $\sigma$), the certified classifier has smaller certified accuracy with small radii, but more samples are certified with larger radii. This trade off is demonstrated in all reported tables.

\input{Sections/supp_tables}

%% file: Sections/supp_tables.tex
\begin{table*}[htp]
\centering
\scriptsize
\caption{\textbf{Rotation Ablations}. The tables show the certified accuracy on MNIST, CIFAR10 and ImageNet for \textsc{DeformRS-Par} while varying the certified rotation radius, \ie $|\theta| \leq r$, for different networks trained with different $\lambda$.}
\begin{tabular}{c| cccccccccc c}
\toprule 
  MNIST  &  {0\degree} &  {10\degree} &  {30\degree} &  {45\degree} &  {60\degree} &  {90\degree} &  {120\degree} &  {135\degree} &  {150\degree}&  {180\degree} &   {\text{ACR}\degree}\\
\midrule
$\lambda = \nicefrac{1\pi}{10}$  &98.81&98.06&0.0&0.0&0.0&0.0&0.0&0.0&0.0&0.0&17.64\\
$\lambda = \nicefrac{2\pi}{10}$  & 98.86&98.37&96.49&0.0&0.0&0.0&0.0&0.0&0.0&0.0&35.10 \\
$\lambda = \nicefrac{3\pi}{10}$   &98.28&97.91&96.85&95.14&0.0&0.0&0.0&0.0&0.0&0.0&52.14
\\
$\lambda = \nicefrac{4\pi}{10}$   &96.59&95.92&94.33&92.61&90.06&0.0&0.0&0.0&0.0&0.0&67.15
 \\
$\lambda = \nicefrac{5\pi}{10}$   & 97.36&97.09&96.28&95.48&94.05&0.0&0.0&0.0&0.0&0.0&85.04
\\
$\lambda = \nicefrac{6\pi}{10}$  &95.55&95.21&94.09&92.95&91.75&87.70&0.0&0.0&0.0&0.0&98.48
  \\
$\lambda = \pi $ & 94.55&94.37&93.85&93.18&92.54&91.31&89.62&88.66&87.52&0.0&163.12\\
\midrule
\midrule
CIFAR10  &  {0\degree} &  {10\degree} &  {30\degree} &  {45\degree} &  {60\degree} &  {90\degree} &  {120\degree} &  {135\degree} &  {150\degree}&  {180\degree} &   {\text{ACR}\degree}\\
\midrule
$\lambda = \nicefrac{3\pi}{10}$& 93.07&91.82&87.77&82.68&0.0&0.0&0.0&0.0&0.0&0.0 &47.17 \\
$\lambda = \nicefrac{4\pi}{10}$& 92.56&91.46&88.75&85.61&80.66&0.0&0.0&0.0&0.0&0.0& 62.17 \\
$\lambda = \nicefrac{5\pi}{10}$ & 91.49&90.67&88.28&86.24&83.20&0.0&0.0&0.0&0.0&0.0 &76.06 \\
$\lambda = \nicefrac{6\pi}{10}$& 91.21&90.55&88.63&87.12&85.03&77.84&0.0&0.0&0.0&0.0 & 90.70 \\
$\lambda = \nicefrac{7\pi}{10}$& 90.64&90.13&88.21&86.87&85.18&80.93&70.83&0.0&0.0&0.0& 117.78  \\
$\lambda = \nicefrac{8\pi}{10}$& 88.88&88.30&87.18&86.00&84.56&81.47&76.14&70.37&0.0&0.0 & 105.05 \\
$\lambda = \nicefrac{9\pi}{10}$ & 88.11&87.62&86.66&85.81&84.76&82.23&78.97&76.13&71.95&0.0 &132.3   \\
$\lambda = \pi $ & 43.28&40.04&33.79&29.31&25.59&18.57&12.72&10.51&8.10&0.0&36.50\\
\midrule
\midrule
ImageNet  &  {0\degree} &  {10\degree} &  {30\degree} &  {45\degree} &  {60\degree} &  {90\degree} &  {120\degree} &  {135\degree} &  {150\degree}&  {180\degree} &   {\text{ACR}\degree}\\
\midrule
$\lambda = \nicefrac{5\pi}{10}$   & 40.60&39.40&36.80&35.20&34.20&29.00&24.20&21.60&19.80&17.00&73.12
 \\
$\lambda = \nicefrac{6\pi}{10}$  & 40.00&39.00&34.80&32.60&30.00&23.00&0.0&0.0&0.0&0.0&32.73
 \\
$\lambda = \nicefrac{7\pi}{10}$   &38.40&37.20&34.20&32.40&30.40&26.00&15.80&0.0&0.0&0.0&36.36
 \\
$\lambda = \nicefrac{8\pi}{10}$  & 36.80&36.20&33.80&32.40&30.80&27.60&21.40&16.40&0.0&0.0&40.52
 \\
$\lambda = \nicefrac{9\pi}{10}$   & 36.80&36.20&34.00&32.40&31.40&28.40&23.60&20.20&17.00&0.0&45.08
  \\
\bottomrule
\end{tabular}\label{tb:rotation-ablation}
\end{table*}


\begin{table*}[t]
\centering
\scriptsize
\caption{\textbf{Scaling Ablations}. The tables show the certified accuracy on MNIST, CIFAR10 and ImageNet for \textsc{DeformRS-Par} while varying the certified scaling radius, \ie $|\alpha-1| \leq r$, for different networks trained with different $\lambda$.}
\centering
\begin{tabular}{c| cccccccccc c}
\toprule 
   MNIST  &  {0\%} &  {10\%} &  {20\%} &  {30\%} &  {40\%} &  {50\%} &  {60\%} &  {70\%} &  {80\%}&  {90\%} &   {\text{ACR}}\\
\midrule
$\lambda =  10\%$  &97.94&0.0&0.0&0.0&0.0&0.0&0.0&0.0&0.0&0.0&0.10
 \\
$\lambda =  20\%$  & 97.98&96.76&0.0&0.0&0.0&0.0&0.0&0.0&0.0&0.0&0.19
\\
$\lambda =  30\%$   &97.96&97.35&95.85&0.0&0.0&0.0&0.0&0.0&0.0&0.0&0.29
\\
$\lambda =  40\%$  &92.84&91.63&89.83&87.64&0.0&0.0&0.0&0.0&0.0&0.0&0.36
\\
$\lambda =  50\%$   &99.00&98.88&98.70&98.36&97.46&0.0&0.0&0.0&0.0&0.0&0.49
 \\
$\lambda =  60\%$  & 97.51&97.04&96.07&95.16&93.48&89.87&0.0&0.0&0.0&0.0&0.56
\\
$\lambda =  70\%$   &98.79&98.70&98.41&97.96&97.13&95.65&91.33&0.0&0.0&0.0&0.67
\\
\midrule
\midrule
   CIFAR10  &  {0\%} &  {10\%} &  {20\%} &  {30\%} &  {40\%} &  {50\%} &  {60\%} &  {70\%} &  {80\%}&  {90\%} &   {\text{ACR}}\\
\midrule
$\lambda =  10\%$  & 93.93&0.0&0.0&0.0&0.0&0.0&0.0&0.0&0.0&0.0&0.09\\
$\lambda =  20\%$  &93.88&91.57&0.0&0.0&0.0&0.0&0.0&0.0&0.0&0.0&0.18
\\
$\lambda =  30\%$   & 94.10&92.48&89.48&0.0&0.0&0.0&0.0&0.0&0.0&0.0&0.27
\\
$\lambda =  40\%$  & 93.72&92.38&90.33&86.76&0.0&0.0&0.0&0.0&0.0&0.0&0.36
\\
$\lambda =  50\%$   &93.34&92.10&90.31&87.70&83.29&0.0&0.0&0.0&0.0&0.0&0.44
 \\
$\lambda =  60\%$  &93.53&92.18&90.53&88.32&84.87&78.77&0.0&0.0&0.0&0.0&0.52
\\
$\lambda =  70\%$   & 92.73&91.76&90.25&88.24&85.00&80.30&71.44&0.0&0.0&0.0&0.58
 \\
\midrule
\midrule
   ImageNet  &  {0\%} &  {10\%} &  {20\%} &  {30\%} &  {40\%} &  {50\%} &  {60\%} &  {70\%} &  {80\%}&  {90\%} &   {\text{ACR}}\\
\midrule
$\lambda =  10\%$  & 51.00&0.0&0.0&0.0&0.0&0.0&0.0&0.0&0.0&0.0&0.05
 \\
$\lambda =  20\%$  & 50.60&44.40&0.0&0.0&0.0&0.0&0.0&0.0&0.0&0.0&0.09
\\
$\lambda =  30\%$   & 49.60&45.60&40.00&0.0&0.0&0.0&0.0&0.0&0.0&0.0&0.13
 \\
$\lambda =  40\%$  &48.20&45.20&40.80&34.00&0.0&0.0&0.0&0.0&0.0&0.0&0.16
\\
$\lambda =  50\%$   &46.80&42.80&38.80&33.60&28.00&0.0&0.0&0.0&0.0&0.0&0.18
\\
$\lambda =  60\%$  & 46.00&43.00&40.00&33.40&27.00&20.00&0.0&0.0&0.0&0.0&0.19
 \\
$\lambda =  70\%$   & 45.00&42.20&39.00&32.80&25.40&20.60&11.20&0.0&0.0&0.0&0.20
\\
\bottomrule
\end{tabular}\label{tb:scaling-ablation}
\end{table*}

\begin{table*}[t]
\centering
\scriptsize
\caption{\textbf{Translation Ablations}. The tables show the certified accuracy on MNIST, CIFAR10 and ImageNet for \textsc{DeformRS-Par} while varying the certified translation radius, \ie $\sqrt{t_u^2 + t_v^2} \leq r$, for different networks trained with different $\sigma$.}
\centering
\begin{tabular}{c| cccccccccc c}
\toprule 
   MNIST  &  {2} &  {4} &  {6} &  {8} &  {10} &  {12} &  {14} &  {16} &  {18}&  {20} &   {\text{ACR}}\\
\midrule
$\sigma =  0.1$  & 98.50&97.73&0.0&0.0&0.0&0.0&0.0&0.0&0.0&0.0&5.23
\\
$\sigma =  0.2$  & 98.88&98.57&97.86&95.01&73.95&0.0&0.0&0.0&0.0&0.0&10.01
\\
$\sigma =  0.3$   &98.94&98.71&98.22&96.50&89.66&69.20&31.26&0.0&0.0&0.0&12.74
\\
$\sigma =  0.4$  & 99.05&98.76&98.14&96.37&89.32&72.99&38.37&7.66&0.15&0.0&13.04
\\
$\sigma =  0.5$   & 99.06&98.68&97.82&94.70&84.84&62.77&29.20&5.48&0.19&0.0&12.48
\\
\midrule
\midrule
   CIFAR10  &  {2} &  {4} &  {6} &  {8} &  {10} &  {13} &  {16} &  {19} &  {22}&  {25} &   {\text{ACR}}\\
\midrule
$\sigma =  0.1$  & 89.88&83.75&73.04&0.0&0.0&0.0&0.0&0.0&0.0&0.0&5.25
\\
$\sigma =  0.2$  & 91.90&88.95&84.61&78.67&69.07&0.0&0.0&0.0&0.0&0.0&9.81
\\
$\sigma =  0.3$   & 92.48&90.35&87.21&82.44&75.79&59.52&32.63&0.0&0.0&0.0&12.71
\\
$\sigma =  0.4$  &92.55&90.46&87.65&83.60&77.94&64.79&43.95&22.24&4.28&0.0&14.03
\\
$\sigma =  0.5$   & 92.20&90.37&87.19&82.80&77.79&64.59&45.99&26.59&9.98&1.52&14.39
\\
\midrule
\midrule
   ImageNet  &  {2} &  {4} &  {6} &  {8} &  {10} &  {13} &  {16} &  {19} &  {22}&  {25} &   {\text{ACR}}\\
\midrule
$\sigma =  0.02$  & 46.80&42.20&35.00&31.80&0.0&0.0&0.0&0.0&0.0&0.0&3.48
\\
$\sigma =  0.03$   & 47.80&44.00&40.20&35.60&32.80&0.0&0.0&0.0&0.0&0.0&5.02
\\
$\sigma =  0.04$  & 49.20&45.40&42.20&38.80&35.80&31.80&27.80&0.0&0.0&0.0&6.57
\\
$\sigma =  0.05$   &48.60&46.00&43.20&41.00&38.20&34.80&31.80&28.40&0.0&0.0&8.10
\\
$\sigma =  0.06$  & 49.20&47.60&45.20&43.20&40.40&36.80&34.80&31.40&28.20&24.40&9.70
\\
\bottomrule
\end{tabular}\label{tb:translation-ablation}
\end{table*}


\begin{table*}[t]
\centering
\scriptsize
\caption{\textbf{Affine Ablations}. The tables show the certified accuracy on MNIST, CIFAR10 and ImageNet for \textsc{DeformRS-Par} while varying the certified affine radius, \ie $\sqrt{a^2+b^2+c^2+d^2+e^2+f^2} \leq r$, for different networks trained with different $\sigma$.}
\centering
\begin{tabular}{c| cccccccccc c}
\toprule 
   MNIST  &  {0.1} &  {0.2} &  {0.3} &  {0.4} &  {0.5} &  {0.6} &  {0.7} &  {0.8} &  {0.9}&  {1.0} &   {\text{ACR}}\\
\midrule
$\sigma =  0.1$  &98.64&97.50&94.67&0.0&0.0&0.0&0.0&0.0&0.0&0.0&0.37

\\
$\sigma =  0.2$  &99.08&98.49&97.31&93.79&80.38&37.58&7.55&0.0&0.0&0.0&0.56

\\
$\sigma =  0.3$   &98.70&97.85&96.15&91.19&76.13&36.39&6.78&0.74&0.0&0.0&0.55

\\
$\sigma =  0.4$  & 98.19&96.88&93.50&82.34&51.44&13.45&0.77&0.0&0.0&0.0&0.49

\\
\midrule
\midrule
  CIFAR10  &  {0.1} &  {0.2} &  {0.3} &  {0.4} &  {0.5} &  {0.6} &  {0.7} &  {0.8} &  {0.9}&  {1.0} &   {\text{ACR}}\\
\midrule
$\sigma =  0.1$  & 89.31&81.23&65.92&0.0&0.0&0.0&0.0&0.0&0.0&0.0&0.30

\\
$\sigma =  0.2$  &90.51&85.50&77.11&64.68&46.46&22.08&2.93&0.0&0.0&0.0&0.44

\\
$\sigma =  0.3$   &88.09&82.81&75.66&65.33&50.84&32.98&15.72&4.59&0.57&0.01&0.46

 \\
$\sigma =  0.4$  & 83.20&77.60&69.53&58.77&45.94&31.07&17.97&8.17&2.87&0.90&0.44

\\
$\sigma =  0.5$   &75.40&69.06&61.15&50.95&39.22&27.88&18.06&9.71&4.83&2.00&0.40

 \\

\midrule
\midrule
   ImageNet  &  {0.02} &  {0.04} &  {0.06} &  {0.08} &  {0.10} &  {0.12} &  {0.14} &  {0.16} &  {0.18}&  {0.20} &   {\text{ACR}}\\
\midrule

$\sigma =  0.03$   & 48.00&43.60&37.20&30.40&26.20&0.0&0.0&0.0&0.0&0.0&0.04
\\
$\sigma =  0.04$  & 49.00&45.20&40.60&35.20&29.60&27.80&24.00&0.0&0.0&0.0&0.06
\\
$\sigma =  0.05$   & 48.60&46.20&42.80&38.40&33.60&30.20&27.40&24.60&19.80&0.0&0.07
\\
$\sigma =  0.06$  &49.00&45.80&44.20&40.60&36.60&32.80&29.20&27.80&24.00&21.20&0.08
\\
\bottomrule
\end{tabular}\label{tb:affine-ablation}
\end{table*}


\begin{table*}
\centering
\scriptsize
\caption{\textbf{DCT Ablations}. The tables show the certified accuracy on MNIST and CIFAR10 for \textsc{DeformRS-Par} while varying the certified DCT radius, \ie $\|\xi\|_2 \leq r$, for different networks trained with different $\sigma$.}
\centering
\begin{tabular}{c| cccccccccc c}
\toprule 
   MNIST  &  {0.1} &  {0.2} &  {0.3} &  {0.4} &  {0.5} &  {0.6} &  {0.7} &  {0.8} &  {0.9}&  {1.0} &   {\text{ACR}}\\
\midrule
$\sigma =  0.1$  & 97.58&91.66&61.55&0.0&0.0&0.0&0.0&0.0&0.0&0.0&0.30
\\
$\sigma =  0.2$  & 96.20&89.89&76.07&49.21&18.83&0.61&0.0&0.0&0.0&0.0&0.38
\\
$\sigma =  0.3$   &92.40&82.18&64.55&44.34&22.52&6.73&0.63&0.0&0.0&0.0&0.36
 \\
$\sigma =  0.4$  &86.93&72.50&51.05&31.33&15.07&6.08&1.99&0.12&0.0&0.0&0.31
\\
$\sigma =  0.5$   & 69.10&54.58&37.08&20.86&12.05&6.36&3.68&1.06&0.04&0.0&0.25
\\
$\sigma =  0.6$   &49.16&35.42&21.58&12.23&7.31&5.09&2.97&0.91&0.03&0.0&0.17
\\
\midrule
\midrule
   CIFAR10  &  {0.1} &  {0.2} &  {0.3} &  {0.4} &  {0.5} &  {0.6} &  {0.7} &  {0.8} &  {0.9}&  {1.0} &   {\text{ACR}}\\
\midrule
$\sigma =  0.1$  & 83.23&66.45&39.66&0.0&0.0&0.0&0.0&0.0&0.0&0.0&0.24
\\
$\sigma =  0.2$  & 79.42&69.47&56.33&40.48&25.17&12.21&3.03&0.0&0.0&0.0&0.33
\\
$\sigma =  0.3$   & 74.47&66.86&57.66&47.74&37.21&27.14&17.73&10.15&4.48&1.15&0.39
\\
$\sigma =  0.4$  &69.92&63.95&56.94&49.68&42.21&34.36&26.43&19.20&12.64&7.24&0.43
\\
$\sigma =  0.5$   & 66.75&61.49&56.07&50.22&44.11&37.86&31.08&24.68&18.53&13.03&0.46
\\
$\sigma =  0.6$   & 62.78&58.25&53.78&48.99&43.72&38.35&32.76&27.06&21.62&16.33&0.47
\\

\bottomrule
\end{tabular}\label{tb:dct-ablation}
\end{table*}


\begin{table*}[t]
\centering
\scriptsize
\caption{\textbf{Vector Field Ablations}. The tables show the certified accuracy on MNIST and CIFAR10 for \textsc{DeformRS-Par} while varying the certified vector field radius, \ie $\|\psi\|_2 \leq r$, for different networks trained with different $\sigma$.}
\centering
\begin{tabular}{c| cccccccccc c}
\toprule 
   MNIST  &  {1.0} &  {2.0} &  {3.0} &  {4.0} &  {5.0} &  {6.0} &  {7.0} &  {8.0} &  {9.0}&  {10} &   {\text{ACR}}\\
\midrule
$\sigma =  0.1$  & 95.41&90.18&79.55&60.31&31.02&0.0&0.0&0.0&0.0&0.0&4.01
\\
$\sigma =  0.2$  &79.26&69.04&55.2&40.37&26.22&15.39&7.99&3.90&1.45&0.28&3.42
\\
$\sigma =  0.3$   & 23.04&16.96&12.29&9.00&6.79&5.40&4.21&3.26&2.49&1.96&1.05
\\
$\sigma =  0.4$  & 15.75&11.74&8.52&5.80&4.12&2.93&2.13&1.40&0.87&0.47&0.64
\\
$\sigma =  0.5$   & 12.96&10.54&8.47&6.79&5.27&4.02&3.03&2.38&1.91&1.48&0.70
\\
$\sigma =  0.6$  & 13.02&11.12&9.51&8.05&6.70&5.65&4.79&3.93&3.16&2.63&0.84
\\
$\sigma =  0.7$   & 11.23&9.96&8.71&7.91&6.86&6.17&5.56&4.79&4.19&3.70&0.92
\\
$\sigma =  0.8$  & 10.85&9.93&9.05&8.42&7.81&7.24&6.62&5.90&5.39&4.92&1.09
\\
$\sigma =  0.9$   & 11.45&10.46&9.62&8.90&8.20&7.50&6.84&6.30&5.74&5.16&1.26
 \\

\midrule
\midrule
   CIFAR10  &  {1.0} &  {2.0} &  {3.0} &  {4.0} &  {5.0} &  {6.0} &  {7.0} &  {8.0} &  {9.0}&  {10} &   {\text{ACR}}\\
\midrule
$\sigma =  0.1$  & 75.83&69.48&63.23&55.67&46.90&32.10&0.0&0.0&0.0&0.0&3.73
\\
$\sigma =  0.2$  & 66.77&63.01&59.40&55.38&51.13&46.98&42.63&37.79&33.47&29.29&5.58
\\
$\sigma =  0.3$   & 61.84&59.09&56.13&53.03&50.02&47.14&44.50&41.56&38.68&35.55&6.98
\\
$\sigma =  0.4$  & 58.56&56.41&54.12&51.67&49.59&47.27&44.98&42.84&40.64&38.36&8.28
\\
$\sigma =  0.5$   & 56.18&54.45&52.36&50.78&49.02&47.23&45.36&43.61&42.11&40.58&9.67
\\
$\sigma =  0.6$  & 54.39&53.07&51.59&50.24&48.85&47.24&45.59&44.06&42.40&40.96&10.96
\\
$\sigma =  0.7$   & 53.37&52.24&51.00&49.73&48.39&47.29&46.07&44.83&43.73&42.35&12.47
\\
$\sigma =  0.8$  &51.48&50.42&49.29&48.24&47.09&46.01&44.96&43.88&42.83&41.74&13.57
\\
$\sigma =  0.9$   & 50.69&49.69&48.72&47.76&46.80&45.70&44.81&43.66&42.83&41.83&14.61\\
\bottomrule
\end{tabular}\label{tb:vf-ablation}
\end{table*}